\documentclass[letterpaper]{article}
\usepackage{aaai20}
\usepackage{times}
\usepackage{helvet}
\usepackage{courier}
\usepackage[hyphens]{url}
\usepackage{graphicx}
\usepackage{graphicx}
\usepackage{pifont}
\usepackage{amsmath}
\usepackage{amsthm}
\usepackage{amsfonts}
\usepackage{amssymb}
\usepackage{algorithm}
\usepackage[noend]{algpseudocode}
\usepackage{algorithmicx}
\usepackage{booktabs}
\usepackage{cleveref}
\usepackage{color}
\usepackage[inline]{enumitem}
\usepackage{etoolbox}
\usepackage{mathtools}
\usepackage{multirow}
\usepackage{tabularx}
\usepackage{thmtools}
\usepackage{thm-restate}
\usepackage{diagbox}
\usepackage{footmisc}

\newcolumntype{C}[1]{>{\centering\let\newline\\\arraybackslash\hspace{0pt}}m{#1}}

\algrenewcommand\algorithmicindent{0.75em}
\floatname{algorithm}{Algorithm}

\makeatletter
\newcommand{\algmargin}{\the\ALG@thistlm}
\makeatother
\newlength{\whilewidth}
\algdef{SE}[parWHILE]{parWhile}{EndparWhile}[1]
{\parbox[t]{\dimexpr\linewidth-\algmargin}{%
		\hangindent\whilewidth\strut\algorithmicwhile\ #1\ \algorithmicdo\strut}}{\algorithmicend\ \algorithmicwhile}%
\algnewcommand{\parState}[1]{\State%
	\parbox[t]{\dimexpr\linewidth-\algmargin}{\strut #1\strut}}

\newtheorem{myeg}{Example}
\AtEndEnvironment{myeg}{\null\hfill\qedsymbol}

\newenvironment{claimproof}[1]{\par\textbf{Proof:}\space#1}{\hfill $\blacksquare$\par}

\newcommand{\citeay}[1]{\citeauthor{#1} (\citeyear{#1})}

\newcommand{\Car}[1]{\Pi_{#1}}
\newcommand{\con}[2]{\nobreak \text{consumers}^{#2}(#1)\allowbreak}
\newcommand{\ex}[1]{\nobreak B^{#1}\allowbreak}

\newcommand{\erp}{MRP}
\newcommand{\epopt}{\text{ex-}\allowbreak\text{post-}\allowbreak\text{efficiency}}
\newcommand{\epopta}{\text{ex-}\allowbreak\text{post-}\allowbreak\text{efficient}}
\newcommand{\eps}{MPS}
\newcommand{\gcycle}{Imp}
\newcommand{\bav}{Top}
\newcommand{\fto}{Ext}

\newcommand{\md}{\mathcal{D}}
\newcommand{\mb}{\mathcal{B}}
\newcommand{\mr}{\mathcal{R}}
\newcommand{\mP}{\mathcal{P}}
\newcommand{\mtap}{\text{MTRA}} 
\newcommand{\mtaps}{\text{MTRAs}} 

\newcommand{\Pa}[1]{Pa(#1)}

\newcommand{\prog}[1]{\nobreak\text{progress}^{#1}\allowbreak}
\newcommand{\ram}{fractional}

\newcommand{\s}{\text{supply}}
\newcommand{\sd}[1]{~\nobreak \succ^{sd}_{#1}\allowbreak~}

\newcommand{\sdef}{\text{sd-}\allowbreak\text{envy-}\allowbreak\text{freeness}}
\newcommand{\sdefa}{\text{sd-}\allowbreak\text{envy-}\allowbreak\text{free}}
\newcommand{\sdopt}{\text{sd-}\allowbreak\text{efficiency}}
\newcommand{\sdopta}{\text{sd-}\allowbreak\text{efficient}}
\newcommand{\sdnopta}{\text{sd-}\allowbreak\text{inefficient}}
\newcommand{\sdsp}{\text{sd-}\allowbreak\text{strategyproofness}}
\newcommand{\sdspa}{\text{sd-}\allowbreak\text{strategyproof}}
\newcommand{\sdwef}{\text{weak-}\allowbreak\text{sd-}\allowbreak\text{envy-}\allowbreak\text{freeness}}
\newcommand{\sdwefa}{\text{weak-}\allowbreak\text{sd-}\allowbreak\text{envy-}\allowbreak\text{free}}
\newcommand{\orfr}{\text{ordinal }\allowbreak\text{fairness}}
\newcommand{\orfra}{\text{ordinally }\allowbreak\text{fair}}
\newcommand{\sdwsp}{\text{weak-}\allowbreak\text{sd-}\allowbreak\text{strategyproofness}}
\newcommand{\sdwspa}{\text{weak-}\allowbreak\text{sd-}\allowbreak\text{strategyproof}}
\newcommand{\upivtr}{\text{upper }\allowbreak\text{invariant }\allowbreak\text{transformation}}
\newcommand{\upiv}{\text{upper }\allowbreak\text{invariance}}
\newcommand{\upiva}{\text{upper }\allowbreak\text{invariant}}
\newcommand{\supply}{\text{supply}}
\newcommand{\topb}[2]{\nobreak top^{#2}(#1)\allowbreak}
\newcommand{\ttd}{\text{\textdagger}}
\newcommand{\ttdbl}{\text{\textdaggerdbl}}
\newcommand{\ucs}{U}			

\newcommand{\vx}{{\bf{x}}}
\newcommand{\vy}{{\bf{y}}}
\newcommand{\vz}{{\bf{z}}}

\newcommand{\egd}{MGD}
\newcommand{\etoe}{\text{equal }\allowbreak\text{treatment }\allowbreak\text{of }\allowbreak\text{equals}}
\newcommand{\decoma}{\text{decomposable}}
\newcommand{\decom}{\text{decomposability}}

\title{Multi-Type Resource Allocation with Partial Preferences}

\author{
	Haibin Wang,\textsuperscript{\rm 1}
	Sujoy Sikdar,\textsuperscript{\rm 2}
	Xiaoxi Guo,\textsuperscript{\rm 1}
	Lirong Xia,\textsuperscript{\rm 3}
	Yongzhi Cao,\textsuperscript{\rm 1\rm *}
	Hanpin Wang\textsuperscript{\rm 4,\rm 1}\\
	\textsuperscript{\rm 1}Key Laboratory of High Confidence Software Technologies (MOE), \\Department of Computer Science and Technology, Peking University, China\\
	\textsuperscript{\rm 2}Computer Science \& Engineering, Washington University in St. Louis\\
	\textsuperscript{\rm 3}Department of Computer Science, Rensselaer Polytechnic Institute\\
	\textsuperscript{\rm 4}School of Computer Science and Cyber Engineering, Guangzhou University, China\\
	beach@pku.edu.cn,
	sujoy@wustl.edu,
    guoxiaoxi@pku.edu.cn,
    xialirong@gmail.com,\\
    caoyz@pku.edu.cn,
    whpxhy@pku.edu.cn,
    \textsuperscript{\rm *}Corresponding Author
}

\begin{document}
\maketitle
\begin{abstract}
We propose  multi-type probabilistic serial (\eps{}) and multi-type random priority (\erp{}) as extensions of the well-known PS and RP mechanisms to the multi-type resource allocation problems (MTRAs) with partial preferences. In our setting, there are multiple types of divisible items, and a group of agents who have partial order preferences over bundles consisting of one item of each type.
We show that for the unrestricted domain of partial order preferences, no mechanism satisfies both \sdopt{} and \sdef{}.
Notwithstanding this impossibility result, our main message is positive:
When agents' preferences are represented by acyclic CP-nets,
\eps{} satisfies \sdopt{}, \sdef{}, \orfr{}, and \upiv{},
while \erp{} satisfies \epopt{}, \sdsp{}, and \upiv{}, recovering the properties of PS and RP.
Besides, we propose a hybrid mechanism, multi-type general dictatorship (\egd{}), combining the ideas of \eps{} and \erp{}, which satisfies \sdopt{}, \etoe{} and \decom{} under the unrestricted domain of partial order preferences.
\end{abstract}

\section{Introduction}

Consider the example of rationing~\cite{Elster92:Local} two types of {\em divisible} resources: food (F) and beverage (B) among two families who have heterogeneous preferences over combinations of food and beverage they wish to consume. For example, a family may prefer water with rice, and milk with wheat. {\em How should we distribute available resources to the families fairly and efficiently?}

In this paper, we consider the problems of divisible {\em multi-type resource allocation problems (\mtaps{})}~\cite{Mackin2016:Allocating} with partial preferences. Here, there are $p\ge 1$ types of $n$ {\em divisible} items per type, with one unit of supply of each item, and a group of $n$ agents with {\em partial preferences} over receiving {\em bundles} consisting of one unit of each type. Our goal is to design mechanisms to fairly and efficiently allocate one unit of items of each type to every agent given their partial preferences over bundles. Such mechanisms are called {\em fractional mechanisms}, because an agent may receive fractions of items.

When there is one type ($p=1$), fractional mechanisms broadly fall in two classes. The random priority (RP) mechanism~\cite{Abdulkadiroglu98:Random} exemplifies the first class: agents consume their favorite remaining item one-by-one according to an ordering drawn from the uniform distribution. Each agent is allocated a fraction of each item equal to the probability that they consume the item. It is easy to see that RP satisfies \epopt{} and equal treatment of equals. Additionally, RP satisfies notions of envy-freeness and strategyproofness through the idea of {\em stochastic dominance (sd)}. Given strict preferences, a fractional allocation $p$ dominates another $q$, if at every item $o$, the total share of $o$ and items preferred to $o$ under $p$ is at least the total share under $q$. RP satisfies \sdwef{} and \sdsp{}~\cite{Bogomolnaia01:New}.

The probabilistic serial (PS) mechanism belongs to the class of ``simultaneous eating'' mechanisms~\cite{Bogomolnaia01:New}.
PS proceeds in multiple rounds. In each round, all agents simultaneously ``eat'' their favorite remaining item at a constant, uniform rate, until one of the items being consumed is exhausted. This terminates when all items are fully consumed, and the output allocates each agent with a fraction of each item that they would consume by this procedure. PS satisfies \sdopt{}, \sdef{}, and \sdwsp{}~\cite{Bogomolnaia01:New}. Besides, PS is the only mechanism that simultaneously satisfies \sdopt, \sdef, and bounded invariance~\cite{Bogomolnaia12:Probabilistic,Bogomolnaia15:Redefining}.

\begin{table*}[ht]
\centering
\caption{Properties of \erp{}, \eps{} and \egd{} under different domain restrictions on partial preferences. A ``Y" indicates that the row mechanism satisfies the column property, and an ``N" indicates that it does not. Results annotated with $\ttd$ are from~\protect\cite{Bogomolnaia01:New}, $\ttdbl$ are from~\protect\cite{Hashimoto14:Two}. Other results are proved in this paper.
}\label{tabl1}
\begin{tabular}{|c|l|*{10}{c|}}
\hline
 \multicolumn{2}{|c|}{\text{\bf Mechanism and Preference Domain}} & \text{\bf SE} & \text{\bf EPE} & \text{\bf OF} & \text{\bf SEF} & \text{\bf WSEF} & \text{\bf ETE} & \text{\bf UI} & \text{\bf SS} & \text{\bf WSS} & \text{\bf DC}\\ \hline
&\text{General partial preferences}
&   N$^\ttd$ & Y & N$^\ttdbl$ & N$^\ttd$ & Y & Y & N & N & Y & Y \\ \cline{2-12}
\text{\erp{}} & \text{CP-nets}
& N$^\ttd$ & Y & N$^\ttdbl$ & N$^\ttd$ & Y & Y & Y & Y & Y & Y \\ \cline{2-12}
& \text{Independent CP-nets}
& N$^\ttd$ & Y & N$^\ttdbl$ & N$^\ttd$ & Y & Y & Y & Y & Y & Y \\ \hline
&\text{General partial preferences}
& Y & N & N & N & Y & Y & N & N$^\ttd$ & N & N \\ \cline{2-12}
\text{\eps{}} & \text{CP-nets}
& Y & N & Y & Y & Y & Y & Y & N$^\ttd$ & N & N \\ \cline{2-12}
&\text{Independent CP-nets}
& Y & N & Y & Y & Y & Y & Y & N$^\ttd$ & Y & N \\ \hline
&\text{General partial preferences}
& Y & Y & N$^\ttdbl$ & N & N & Y & N & N & N & Y \\ \cline{2-12}
\text{\egd{}} & \text{CP-nets}
& Y & Y & N$^\ttdbl$ & N & N & Y & N & N & N & Y \\ \cline{2-12}
&\text{Independent CP-nets}
& Y & Y & N$^\ttdbl$ & N & N & Y & N & N & N & Y \\ \hline
\end{tabular}
\end{table*}

Our work is the first to consider the design of fair and efficient mechanisms for \mtaps{} with partial preferences, and the first to extend fractional mechanisms to \mtaps{} with partial preferences, to the best of our knowledge.~\citeay{Katta06:Solution} mention that PS can be extended to partial orders but we are not aware of a (formal or informal) work that explicitly defines such an extension and studies its properties.~\citeay{Monte2015:Centralized} and~\citeay{Mackin2016:Allocating} consider the problem of \mtaps{} under linear preferences, but do not fully address the issue of fairness.~\citeay{Ghodsi11:Dominant} consider the problem of allocating multiple types of resources, when the resources of each type are indistinguishable, and agents have different demands for each type of resources. However, the problem of finding fair and efficient assignments for \mtaps{} with partial preferences remains open.

Our mechanisms output {\em fractional assignments}, where each agent receives a fractional share of bundles consisting of an item of each type, which together amount to one unit per type. The fractional assignments output by our mechanisms also specify for each agent how to form bundles for consumption from the assigned fractional shares of items. Our setting may be interpreted as a special case of {\em cake cutting}~\cite{Brams96:Fair,Brams06:Better,Procaccia13:Cake}, where the cake is divided into parts of unit size of $p$ types, and $n$ parts per type, and agents have complex combinatorial preferences over being assigned combinations of parts of the cake which amount to one unit of each type.

\noindent{\bf Our Contributions.}
Our work is the first to provide fair and efficient mechanisms for \mtaps{}, and the first to extend PS and RP both to \mtaps{} and to partial preferences, to the best of our knowledge. We propose  multi-type probabilistic serial (\eps{}) and multi-type random priority (\erp{}) as the extensions of PS and RP to \mtaps{}, respectively. Our main message is positive: Under the well-known and natural domain restriction of CP-net preferences~\cite{Boutilier04:CP}, \erp{} and \eps{} satisfy all of the fairness and efficiency properties of their counterparts for single types and complete preferences.

Unlike single-type resources allocations, in \mtaps{}, not all fractional assignments are \decoma{}, where assignments can be represented as a probability distribution over assignments where each agent receives a bundle consisting of whole items. Unfortunately, the output of \eps{} may be indecomposable. In response to this, we propose a new mechanism, multi-type general dictatorship (\egd{}), which is \decoma{} and matches the efficiency of \eps{}, and satisfies equal treatment of equals.

Our technical results are summarized in Table~\ref{tabl1}. We extend {\em stochastic dominance} to compare two fractional allocations under partial preferences. Here, a fractional allocation $p$ is said to stochastically dominate another allocation $q$ w.r.t. an agent's partial preference, if at any bundle, the fractional share of weakly dominating bundles in $p$ is larger than or equal to the fractional shares of the bundles in $q$ according to her preference. Formal definitions of stochastic dominance and properties in Table~\ref{tabl1} can be found in Preliminaries.

For the unrestricted domain of general partial preferences, unfortunately, no mechanism satisfies both \sdopt{} (SE) and \sdef{} (SEF) as we prove in Proposition~\ref{prop:noboth}. Despite this impossibility result, \erp{}, \eps{} and \egd{} satisfy several desirable properties: We show in
\begin{itemize}[leftmargin=*,wide,labelindent=0pt,topsep=\parskip,itemsep=-1pt]
    \item[-] Theorem~\ref{thm:erp} that \erp{} satisfies \epopt{} (EPE), \sdwef{} (WSEF), \etoe{} (ETE), \sdwsp{} (WSS), and decomposibility (DC);
    \item[-] Theorem~\ref{thm:eps} that \eps{} satisfies \sdopt{} (SE), \etoe{} (ETE), \sdwef{} (WSEF);
    \item[-] Theorem~\ref{thm:egd} that \egd{} satisfies \sdopt{} (SE), \epopt{} (EPE), \etoe{} (ETE), and decomposibility (DC).
\end{itemize}

Remarkably, we recover the fairness properties of \eps{}, and the truthfulness and invariance properties for \erp{} and \eps{} under the well-known and natural domain restriction of acyclic CP-net preferences~\cite{Boutilier04:CP}. We show in:
\begin{itemize}[leftmargin=*,wide,labelindent=0pt,topsep=\parskip,itemsep=-1pt]
    \item[-] Theorem~\ref{thm:cperp}, that \erp{} is \sdspa{} (SS);

    \item[-] Theorem~\ref{thm:cpepsef}, that \eps{} is \sdefa{} (SEF) and \orfra{} (OF);

    \item[-] Proposition~\ref{prop:cpepsuiv} that \eps{} is \upiva{} (UI); and

    \item[-] Proposition~\ref{prop:cpepswsp} that \eps{} is \sdwspa{} (WSS) under the special case where agents' CP-nets are independent between each type.
\end{itemize}

\noindent{\bf Discussions.}
\egd{} may be viewed as a hybrid between \erp{} and \eps{}: on the one hand as a modification of \erp{} where the priorities depend on the preference profile, instead of being drawn from the uniform distribution; and on the other hand, as a modification of \eps{} where only a subset of the agents are allowed to eat simultaneously in each round. The design of~\egd{} is motivated by the natural desire for decomposability (not satisfied by \eps{}), while maintaining the stronger efficiency notion of \sdopt{} (not satisfied by \erp{}), and the basic fairness property of equal treatment of equals. Decomposable assignments are desired when sharing of items imposes overhead costs, as noted by~\cite{Sandomirskiy19:Fair}. However, \begin{enumerate*}[label=(\roman*)]\item for \mtaps{}, fractional assignments are not guaranteed to be \decoma{} as we show in assignment (\ref{m2}) of Example~\ref{eg:sd}, and indeed, \eps{} does not guarantee \decoma{} assignments; \item fractional mechanisms are inevitable when certain fairness guarantees are desired. Even when there is one type ($p=1$), no mechanism which assigns each item fully to a single agent, can satisfy the basic fairness property of equal treatment of equals, whereby, everything else being equal, agents with the same preferences should receive the same share of the resources (e.g.~two agents having identical strict preferences).\end{enumerate*} On the flip side, \egd{} does not satisfy \sdwef{}, while \erp{} and \eps{} both do.

\section{Related Work}
\mtaps{} belong to a long line of research on mechanism design for multi-agent resource allocation (see ~\cite{Chevaleyre06:Issues} for a survey), where the literature focuses on the settings with a single type of items.~\citeay{Mackin2016:Allocating} characterize serial dictatorships for \mtaps{} by strategyproofness, neutrality, and non-bossiness. The exchange economy of multi-type housing markets~\cite{Moulin95:Cooperative} is considered in~\cite{Sikdar2017:Mechanism,Sikdar18:Top} under lexicographic preferences, while~\citeay{Fujita2015:A-Complexity} consider the exchange economy where agents may consume multiple units of a single type of items under lexicographic preferences.

Our work is the first to extend RP and PS under partial preferences, to the best of our knowledge despite the vast literature on fractional assignments.
~\citeay{Hosseini19:Multiple} consider RP under lexicographic preferences. The remarkable properties of PS has encouraged extensions to several settings.~\citeay{Hashimoto14:Two} provide two characterizations of PS: (1) by \sdopt{}, \sdef{}, and \upiv{}, and (2) by \orfr{} and non-wastefulness. In~\cite{Heo14:Probabilistic,Hatfield09:Strategy-proof}, there is a single type of items, and agents have multi-unit demands. In~\cite{Saban14:Note}, the supply of items may be different, while agents have unit demand and are assumed to have lexicographic preferences. Other works extend PS to settings where indifference relationships are  allowed~\cite{Katta06:Solution,Heo15:Characterization}.~\citeay{Aziz15:Fair} consider fair assignments when indifference allowed in preferences (but not incomparabilities).~\citeay{Yilmaz2009:Random}, ~\citeay{Athanassoglou11:House} extend PS to the housing markets problem~\cite{Shapley74:Cores}.~\citeay{Bouveret10:Fair} study the complexity of computing fair and efficient allocations under partial preferences represented by SCI-nets for allocation problems with a single type of indivisible items.

\subsection{CP-net Preferences}
Compact preference representations are a common approach to deal with the preference formation and elicitation bottleneck faced in \mtaps{}, where the number of bundles grows exponentially with the number of types. CP-nets~\cite{Boutilier04:Preference} are perhaps the most well-studied and natural compact preference representation language allowing agents to express conditional (in)dependence of their preferences over combinations of different types (see Example~\ref{eg:partial}). CP-nets are an important restriction on the domain of partial preferences, and induce a partial ordering on the set of all bundles.~\citeay{Sikdar2017:Mechanism} design mechanisms for multi-type housing markets under lexicographic extensions of CP-nets. Several works in the combinatorial voting literature assume CP-net preferences~\cite{Rossi04:mCP,Lang07:Vote}, and that agents' CP-nets have a common dependence structure (see~\cite{Lang16:Voting} for a recent survey).

\section{Preliminaries}
A {\em multi-type resource allocation problem (\mtap{})}~\cite{Mackin2016:Allocating}, is given by a tuple $(N,M,R)$. Here, \begin{enumerate*}[label=(\arabic*)]\item $N=\{1,\dots,n\}$ is a set of agents. \item $M=D_1\cup\dots\cup D_p$ is a set of items of $p$ types, where for each $i\le p$, $D_i$ is a set of $n$ items of type $i$, and there is one unit of {\em \supply{}} of each item in $M$. We use $\md=D_1\times\dots\times D_p$ to denote the set of {\em bundles}. \item $R=(\succ_j)_{j\le n}$ is a {\em preference profile}, where for each $j\le n$, $\succ_j$ represents the preference of agent $j$, and $R_{-j}$ represents the preferences of agents in $N\setminus\{j\}$. We use $\mr$ to denote the set of all possible preference profiles. \end{enumerate*}

\noindent{\bf Bundles.} For any type $i\le p$, we use $k_i$ or $k_t$ to refer to the $k$-th item of type $i$ where $t$ represents the name of type $i$. Each bundle $\vx\in\md$ is a $p$-tuple, and we use $o \in \vx$ to indicate that bundle $\vx$ contains item $o$. We define $T=\{D_1,\dots,D_p\}$, and for any $S\subseteq T$, we define $\Car{S}=\bigtimes_{D\in S}D$, and $-S=T\setminus S$. For any $S\subseteq T$, $\hat S\subseteq T\setminus S$, and any $\vx\in\Car{S}, \vy\in\Car{\hat S}$, $(\vx,\vy)$ denotes the bundles consisting of all items in $\vx$ and $\vy$. For any $S\subseteq T$, $D\in T\setminus S$, and any $\vx\in S, o\in D$, $(o,\vx)$ denotes the bundles consisting of $o$ and the items in $\vx$.

\noindent{\bf Partial Preferences and Profiles.} A partial preference $\succ$ is a partial order over $\md$, which is an irreflexive, anti-symmetric, and transitive binary relation. Given a partial preference $\succ$ over $\md$, we define the corresponding {\em preference graph}, denoted by $G_\succ$, to be the directed graph whose nodes are the bundles in $\md$, and for every $\vx,\vy\in\md$, there is an directed edge $(\vx,\vy)$ if and only if $\vx\succ\vy$ and there exists no $\vz\in\md$ such that $\vx\succ\vz$ and $\vz\succ\vy$. Given a partial order $\succ$ over $\md$, we define the {\em upper contour set} of $\succ$ at a bundle $\vx\in\md$ as $\ucs(\succ,\vx)=\{\hat\vx:\hat\vx \succ \vx \text{ or } \hat\vx=\vx \}$.

\noindent{\bf Acyclic CP-nets.} A CP-net~\cite{Boutilier04:CP} $\succ$ over the set of variables $\md$ has two parts:
\begin{enumerate*}[label=(\roman*)]
	\item a directed graph $G=(T,E)$ called the \emph{dependency} graph, and
	\item for each $i\le p$, there is a {\em conditional preference table} $CPT(D_i)$ that contains a linear order $\succ^{\vx}$ over $D_i$ for each $\vx \in \Car{\Pa{D_i}}$, where $Pa(D_i)$ is the set of types corresponding to the parents of $D_i$ in $G$.
\end{enumerate*}
When $G$ is (a)cyclic we say that $\succ$ is a (a)cyclic CP-net. The partial order induced by an acyclic CP-net $\succ$ over $\md$ is the transitive closure of $\{(o,\vx,\vz) \succ (\hat o,\vx,\vz): i \le p; o,\hat o \in D_i; o \succ^{\vx} \hat o; \vx \in \Car{\Pa{D_i}}; \vz \in \Car{-(\Pa{D_i} \cup \{D_i\})}\}$. A CP-profile is a profile of agents' preferences, each of which is represented by an acyclic CP-net. We say the acyclic CP-net with a trivial dependency graph such that there is no edge in the dependency graph as {\em independent CP-net}. An {\em independent CP-profile} is a profile of agents' preferences, each of which is represented by an independent acyclic CP-net.

\begin{figure}[h]
    \centering
    \includegraphics[width=0.45\textwidth]{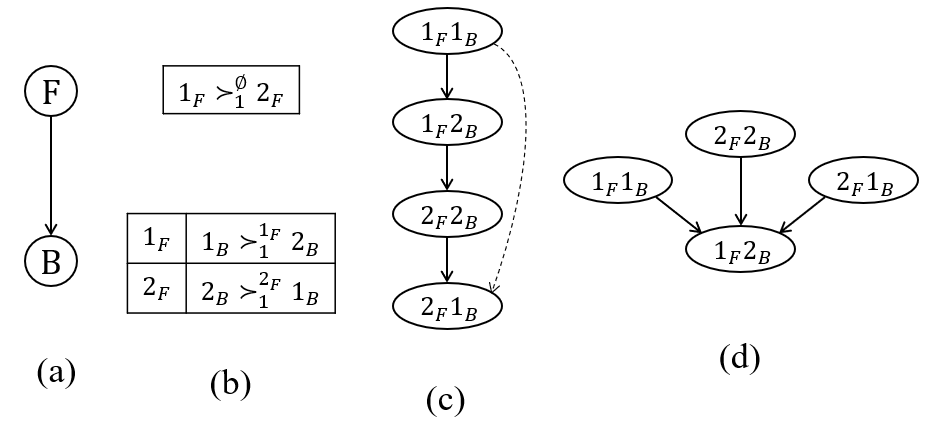}
    \caption{Agent $1$'s preferences are represented by an acyclic CP-net with dependence graph (a), and CP-tables (b), whose preference graph is in (c). Agent $2$'s preferences are represented by the preference graph in (d).}
    \label{fg:popg1}
\end{figure}

\begin{myeg}\label{eg:partial}
	Consider \mtap{} $(N,M,R)$ with $p = 2$ types, food (F) and beverage (B), where $N=\{1,2\}, M=\{1_F,2_F,1_B,2_B\}$, where $1_B$ is item $1$ of type B and so on. Let agent $1$'s preference $\succ_1$ be represented by the acyclic CP-net in Figure~\ref{fg:popg1}, where the dependency graph (Figure~\ref{fg:popg1} (a)) shows that her preference on type B depends on her assignment in type F. The corresponding conditional preference tables (Figure~\ref{fg:popg1} (b)) show that agent $1$ prefers $1_B$ with $1_F$, and $2_B$ with $2_F$. This induces the preference graph in Figure~\ref{fg:popg1} (c) which happens to be a linear order.
	Let agent $2$'s preference $\succ_2$ be represented by the preference graph in Figure~\ref{fg:popg1} (d) which represents a partial order, where $1_F2_B$ is the least preferred bundle, and providing no information on the relative ordering of other bundles.
\end{myeg}

\noindent{\bf Assignments.} A {\em discrete assignment} $A:N\to\md$ is a one to one mapping from agents to bundles such that no item is assigned to more than one agent. A {\em \ram{} allocation} shows the fractional shares an agent acquires over $\md$, represented by a vector $p = [p_\vx]_{\vx\in\md}$, $p\in [0,1]^{1\times|\md|}$ such that $\sum_{\vx\in\md}p_\vx = 1$. We use $\Pi$ to denote the set of all possible \ram{} allocations on an agent. A {\em \ram{} assignment} $P$ is a combination of all agents' \ram{} allocations, and can be represented by a matrix $P=[p_{j,\vx}]_{j\le n, \vx\in\md}$, $P\in[0,1]^{|N|\times|\md|}$, such that
\begin{enumerate*}[label=(\roman*)]
	\item for every $j\le n$, $\sum_{\vx\in\md}p_{j,\vx}=1$,
	\item for every $o\in M$, $S_o=\{\vx:\vx\in\md \text{ and }o\in\vx\}$, $\sum_{j\le n,\vx \in S_o}p_{j,\vx}=1$.
\end{enumerate*}
The $j$-th row of $P$ represents agent $j$'s \ram{} allocation under $P$, denoted $P(j)$. We use $\mP$ to denote the set of all possible \ram{} assignments.

\noindent{\bf Mechanisms.} A {\em mechanism} $f:\mr\to\mP$ is a mapping from profiles to \ram{} assignments. For any profile $R\in\mr$, we use $f(R)$ to refer to the \ram{} assignment output by $f$, and for any agent $j\le n$ and any bundle $\vx\in\md$, we use $f(R)_{j,\vx}$ to refer to the value of the element of the matrix $f(R)$ indexed by $j$ and $\vx$.

\noindent{\bf Stochastic Dominance. }
In this paper we will use the natural extension of {\em stochastic dominance}~\cite{Bogomolnaia01:New} to compare fractional allocations.
\begin{restatable}{dfn}{dfnsd}{\bf(stochastic dominance for fractional allocations)}\label{dfn:sd}
	Given a partial preference $\succ$ over $\md$, the {\em stochastic dominance} relation associated with $\succ$, denoted $\sd{\null}$ is a weak ordering over $\Pi$ such that for any pair of fractional allocations $p,q\in\Pi$, $p$ {\em stochastically dominates} $q$, denoted $p\sd{\null} q$, if and only if for every $\vx\in\md$, $\sum_{\hat\vx\in\ucs(\succ,\vx)}p_{\hat\vx}\ge\sum_{\hat\vx\in\ucs(\succ,\vx)}q_{\hat\vx}$.
\end{restatable}

We use $P\sd{j} Q$ to denote $P(j)\sd{j} Q(j)$. We write $P \sd{\null} Q$ if for every $j\le n$, we have $P\sd{j} Q$.

\begin{myeg}\label{eg:sd}
Consider the situation in Example~\ref{eg:partial} with the three \ram{} assignments (\ref{m3}), (\ref{m1}), and (\ref{m2}) shown below. The four upper contour sets for agent $1$ are $\{1_F1_B\}$, $\{1_F1_B,1_F2_B\}$, $\{1_F1_B,1_F2_B,2_F2_B\}$, and $\{1_F1_B,1_F2_B,2_F1_B,2_F2_B\}$. The allocations of the four upper contour sets for agent $1$ in assignment (\ref{m1}) are 0.5, 0.5, 1, 1, respectively and in assignment (\ref{m2}) are 0, 0.5, 0.5, 1. The former is greater than or equal to the latter respectively, and the same can be concluded for agent $2$ by considering the upper contour sets $\{1_F1_B\}$, $\{2_F2_B\}$, $\{2_F1_B\}$, and $\{1_F1_B,1_F2_B,2_F1_B,2_F2_B\}$. Hence assignment (\ref{m1}) stochastically dominates (\ref{m2}). We can check that (\ref{m1}) does not dominate (\ref{m3}), and that the reverse is true by considering the sets of $\{1_F1_B\}$ and $\{2_F1_B\}$ for agent $2$.
 	\begin{equation}\label{m3}
	\begin{tabular}{lcccc}\toprule
	&  $1_F1_B$   & $1_F2_B$  & $2_F1_B$    & $2_F2_B$  \\
	\midrule
\text{Agent	1} &  0.5  & 0.5   & 0     & 0    \\
\midrule
\text{Agent	2}	&  0    & 0 & 0.5   & 0.5\\
\bottomrule
	\end{tabular}
	\end{equation}
	\begin{equation}\label{m1}
	\begin{tabular}{lcccc}
	\toprule
	&  $1_F1_B$   & $1_F2_B$  & $2_F1_B$    & $2_F2_B$  \\
	\midrule
\text{Agent	1} &  0.5  & 0   & 0     & 0.5    \\
	\midrule
\text{Agent	2} &  0.5  & 0   & 0     & 0.5 \\
	\bottomrule
	\end{tabular}
	\end{equation}
	\begin{equation}\label{m2}
	\begin{tabular}{lcccc}
	\toprule
	&  $1_F1_B$   & $1_F2_B$  & $2_F1_B$    & $2_F2_B$  \\
	\midrule
\text{Agent	1} &  0  & 0.5   & 0.5     & 0    \\
	\midrule
\text{Agent	2} &  0.5    & 0   & 0     & 0.5\\
	\bottomrule
	\end{tabular}
	\end{equation}
\end{myeg}
\noindent{\bf Desirable Properties.} A \ram{} assignment $P$ satisfies:
\begin{enumerate*}[label=(\roman*)]
	\item {\bf \sdopt{}}, if there is no \ram{} assignment $Q\neq P$ such that $Q \sd{\null} P$,
	\item {\bf \epopt{}}, if $P$ can be represented as a probability distribution over \emph{\sdopta{}} discrete assignments,
	\item {\bf \sdef{}}, if for every pair of agents $j,\hat j\le n$, $P(j)\sd{j}P(\hat j)$,
	\item {\bf \sdwef{}}, if for every pair of agents $j,\hat j\le n$, $P(\hat j)\sd{j}P(j) \implies P(j)=P(\hat j)$,
	\item {\bf \etoe{}}, if for every pair of agents $j,\hat j\le n$ such that agents $j$ and $\hat j$ have the same preference, $P(j)=P(\hat j)$,
	\item {\bf \orfr{}}, if for every bundle $\vx \in \md$ and every pair of agents $j,\hat j\le n$ with $P_{j,\vx}>0$, $\sum_{\hat\vx\in \ucs(\succ_j, \vx)}P_{j,\hat\vx} \leq \sum_{\hat\vx\in \ucs(\succ_{\hat j}, \vx)}P_{\hat j,\hat\vx}$, and
	\item {\bf \decom{}}, if $P$ can be represented as a probability distribution over discrete assignments.
\end{enumerate*}

A mechanism $f$ satisfies $X\in$ \{\sdopt{}, \epopt{}, \sdef{}, \sdwef{}, \etoe{}, \orfr{}, \decom{}\}, if for every $R\in\mr$, $f(R)$ satisfies $X$. A mechanism $f$ satisfies:
\begin{enumerate*}[label=(\roman*)]
\item {\bf \sdsp{}} if for every profile $R\in\mr$, every agent $j\le n$, every $R'\in\mr$ such that $R'=(\succ'_j,\succ_{-j})$, it holds that $f(R)\sd{j}f(R')$, and
\item {\bf weak-sd-strategyproofness} if for every profile $R\in\mr$, every agent $j\le n$, every $R'\in\mr$ such that $R'=(\succ'_j,\succ_{-j})$, it holds that $f(R')\sd{j}f(R)\implies f(R')(j)=f(R)(j)$.
\end{enumerate*}

Given any partial preferences $\succ$, we denote $\succ\mid_\mb$ by the restriction of $\succ$ to $\mb \subseteq \md$, i.e., $\succ\mid_\mb$ is a preference relation over $\mb$ such that for all $\vx,\vy\in\mb$, $\vx \succ\mid_\mb \vy \Leftrightarrow \vx \succ \vy$. Then for any $j\le n$, $\succ'_j$ is an {\bf \upivtr{}} of $\succ_j$ at $\vx\in\md$ under a \ram{} assignment $P$ if for some $\mathcal{Z}\subseteq \{\vy \in \md \mid P_{j,\vy} = 0\}$, $\ucs(\succ'_j, \vx) = \ucs(\succ_j, \vx)\setminus \mathcal{Z}$ and $\succ'_j\mid_{\ucs(\succ'_j,\vx)} = \succ_j\mid_{\ucs(\succ'_j,\vx)}$. A mechanism $f$ satisfies {\bf \upiv{}} if it holds that $f(R)_{\hat j, \vx} = f(R')_{\hat j, \vx}$ for every $\hat j \le n$, $j\le n$, $R\in\mr$, $R'\in\mr$, and $\vx\in \md$,  such that $R'=(\succ'_j,\succ_{-j})$ and $\succ'_j$ is an \upivtr{} of $\succ_j$ at $\vx$ under $f(R)$.

\section{Mechanisms for \mtaps{} with Partial Preferences}
In this section, we propose \erp{} (Algorithm~\ref{alg:erp} as extension of RP), \eps{} (Algorithm~\ref{alg:eps}, as extension of PS), and \egd{} (Algorithm~\ref{alg:egd}), which can be seen as not only an eating algorithm but a special random priority algorithm.

The three mechanisms operate on a modified preference profile of strict preferences, where for every agent with partial preference $\succ$, an arbitrary deterministic topological sorting is applied to obtain a strict ordering $\succ'$ over $\md$, such that for any pair of bundles $\vx, \vy \in \md$, $\vx \succ \vy \implies \vx \succ' \vy$.
 Given a strict order $\succ'$ obtained in this way, and remaining $M'$, we use $\fto(\succ',M')$ to denote the first available bundle in $\succ'$, which we refer to as the agents' favorite bundle. It is easy to see that no available bundle is preferred over $\fto(\succ',M')$ according to $\succ$.

 Our results apply to arbitrary deterministic topological sortings (induced by a fixed ordering over items), even though different topological sortings may lead to different outputs.

\begin{algorithm}[h]
    \caption{\label{alg:erp}\erp{}}
    \begin{algorithmic}
    \item[\textbf{Input:}] An \mtap{} $(N,M,R)$
    \item [\textbf{Output:}] Assignment $P$
    \end{algorithmic}
    \begin{algorithmic}[1]
    \State For each $j\leq n$, compute a linear ordering $\succ'_j$ corresponding to a deterministic topological sort of $G_{\succ_j}$.
    \State $P \gets 0^{|N|\times|\md|}$ and $M' \gets M$.
    \State Pick a random priority order $\rhd$ over agents.
    \State Successively pick a highest priority agent $j^*$ according to $\rhd$. $\vx^* \gets \fto(\succ'_{j^*},M')$ and set $P_{j^*,\vx^*} \gets 1$. Remove $j^*$, and remove all items contained by $\vx^*$ in $M'$.
    \State \Return{$P$}
    \end{algorithmic}
\end{algorithm}

Given an instance of \mtap{} with agents' partial preferences, \erp{} fixes an arbitrary deterministic topological sorting $\succ'$ of agents' preferences, and sorts agents uniformly at random. Then agents get one unit of their favorite available bundle from the remaining $M'$ in turns as in RP.

\begin{algorithm}[h]
    \caption{\label{alg:eps}\eps{}}
    \begin{algorithmic}
    \item[\textbf{Input:}] An \mtap{} $(N,M,R)$
    \item [\textbf{Output:}] Assignment $P$
    \end{algorithmic}
    \begin{algorithmic}[1]
    \State  For each $j\leq n$, compute a linear ordering $\succ'_j$ corresponding to a deterministic topological sort of $G_{\succ_j}$.
    \State $P\gets 0^{|N|\times|\md|}$ and $M' \gets M$. For every $o\in M$, $\s(o)\gets 1$, $\ex{}\gets\emptyset$, $\prog{}\gets 0$.
	\While{$M'\neq\emptyset$}
	\State $\topb{j}{} \gets \fto(\succ'_j,M')$ for every agent $j\le n$.
	\parState{ {\bf Consume.}
		\begin{enumerate}[label=5.\arabic*:,leftmargin=*,topsep=\parskip]
			\item For each $o\in M'$, $\con{o}{}\gets|\{j\in N:o\text{ is in }\topb{j}{}\}|$.
			\item $\prog{}\gets \min_{o\in M'}\frac{\s(o)}{\con{o}{}}$.
			\item For each $j\le n$, $P_{j,\topb{j}{}}\gets P_{j,\topb{j}{}}+\prog{}$
			\item For each $o\in M'$, $\s^{}(o)\gets\s^{}(o)-{\prog{}\times \con{o}{}}$.
		\end{enumerate}
	}
	\State $\ex{}\gets\arg\min_{o\in M'}\frac{\s^{}(o)}{\con{o}{}}$, $M'^{}\gets M'^{}\setminus B^{}$
	\EndWhile
	\State \Return $P$
    \end{algorithmic}
\end{algorithm}

Given an instance of \mtap{} with agents' partial preferences, \eps{} involves applying the PS mechanism to a modified profile $\succ'$ over $\md$ using an arbitrary deterministic topological sorting in multiple rounds as follows. In each round, each agent consumes their favorite {\em available} bundle by consuming each item in the bundle at an uniform rate of one unit of an item per type per unit of time, until one of the bundles being consumed becomes unavailable because the \supply{} for one of the items in it is exhausted.

\begin{algorithm}[h]
    \caption{\label{alg:egd}\egd{}}
    \begin{algorithmic}
    \item[\textbf{Input:}] An \mtap{} $(N,M,R)$
    \item [\textbf{Output:}] Assignment $P$
    \end{algorithmic}
    \begin{algorithmic}[1]
    \State For each $j\leq n$, compute a linear ordering $\succ'_j$ corresponding to a deterministic topological sort of $G_{\succ_j}$.
    \State $P \gets 0^{|N|\times|\md|}$ and $M' \gets M$.
    \For{ $j = 1$ to $n$}
    	\State $\topb{j}{} \gets \fto(\succ'_j,M')$.
    	\State For each $j'\in Group(j,\succ')$, let $$P_{j',\topb{j}{}}\gets \frac{1}{|Group(j,\succ')|}$$
    	\State $M'^{}\gets M'^{}\setminus \{o\in M:o\text{ is in }\topb{j}{}\}$.
    \EndFor
    \State \Return{$P$}
    \end{algorithmic}
\end{algorithm}

Given an instance of \mtap{} with agents' partial preferences, \egd{} proceeds by operating on an arbitrary deterministic topological sorting $\succ'$. Let $Group(j, \succ')$ denote the set of all agents who have the same order with agent $j$ in $\succ'$. \egd{} proceeds in $n$ rounds as follows. In each round $j \leq n$, agent $j$ comes and invites other agents with the same topological sort $\succ'_j$ to consume her favorite available bundle $\fto(\succ'_j,M')$ with $\frac{1}{|Group(j,\succ')|}$ unit of time and one unit of eating rate. We present \egd{} as its eating algorithm version here and show that $\egd{}(R)$ is the expected result of a special random priority algorithm when we prove its \decom{}.

\begin{myeg}\label{eg:mwork}
Consider the situation in Example~\ref{eg:partial}, the topological sort of agent $1$'s preference is $1_F1_B \succ' 1_F2_B \succ' 2_F2_B \succ' 2_F1_B$ but both $2_F1_B \succ' 1_F1_B \succ' 2_F2_B \succ' 1_F2_B$ and $1_F1_B \succ' 2_F2_B \succ' 2_F1_B \succ' 1_F2_B$ can be the topological sort of agent $2$'s preference. If the topological sort of agent $2$ is the former, in \eps{}, agent $1$ consumes $1_F1_B$ and agent $2$ consumes $2_F1_B$ at the beginning. When they both consume $0.5$ fraction, $1_B$ is exhausted. $1_F1_B$ and $2_F1_B$ become unavailable. They turn to identify the next bundles in line $4$. For example, the first two bundles $2_F1_B$ and $1_F1_B$ are unavailable for agent $2$, so she turns to the third bundle $2_F2_B$. Agent $1$ turns to $1_F2_B$. Then \eps{} goes to next round of consumption until all left items are exhausted at the same time. The result is shown in assignment (\ref{m3}) in Example~\ref{eg:sd}. However, if the topological sort of agent $2$ is the latter, agent $1$ and agent $2$ both get $0.5$ of $1_F1_B$ and $0.5$ of $2_F2_B$ as assignment (\ref{m1}) shows. It is easy to check that \erp{} has the same conclusion. Suppose agent $1$ has the same partial preference with agent $2$ in Example~\ref{eg:partial} in \egd{}. If their topological sorts are both the former, agent $1$ invites agent $2$ to consume $2_F1_B$ in round 1 and agent $2$ invites agent $1$ to consume $1_F2_B$ in round 2. But if their topological sorts are both the latter, agent $1$ invites agent $2$ to consume $1_F1_B$ in round 1 and agent $2$ invites agent $1$ to consume $2_F2_B$ in round 2.
\end{myeg}

There is an unique best available bundle w.r.t. any acyclic CP-net preference and remaining supply of items, which can be computed in polynomial time by induction on the types according to the dependency graph as we show in Proposition~\ref{prop:acyclicCP}. This is an extension of the well-known result of~\citeay{Boutilier04:CP} that there is a unique best bundle w.r.t. any acyclic CP-net.

\begin{restatable}{prop}{propacyclicCP}\label{prop:acyclicCP}
	Let $T'=\{D'_1\subseteq D_1,\dots,D'_p\subseteq D_p\}$, $\md'=\Car{T'}$, and let $\succ$ be any acyclic CP-net over $\md$. Then, there exists unique $\vx\in\md'$ such that for every $\vy\neq\vx\in\md'$, $\vx\succ\vy$.
\end{restatable}

All missing proofs can be found in Appendix.

We use $\bav(\succ,M')$ to denote the best available bundle in remaining $M'$ given an acyclic CP-net $\succ$. Under the domain restriction of acyclic CP-net preferences, for any topological sorting algorithms, $\fto(\succ'_j,M')$ is exact $\bav(\succ_j,M')$ for any $M'$ and $j\leq n$ by Proposition~\ref{prop:acyclicCP}. Therefore, we can remove line $1$ and instead $\fto(\succ'_j,M')$ with $\bav(\succ_j,M')$ in the three mechanisms to save the time and space.

\begin{restatable}{prop}{propcost}\label{prop:cost}
	\erp{}, \eps{} and \egd{} run in $O(n^{p+1})$ time.
\end{restatable}
We note that the size of the preference representation is $O(n^{p+1})$, and forms a part of the input.

\section{Properties under General Partial Preferences}

\begin{restatable}{thm}{thmerp}\label{thm:erp}
    Under general partial preferences, $\erp{}$ satisfies \epopt{}, \sdwef{}, \etoe{}, \sdwsp{}, and decomposability.
\end{restatable}

The proof of \sdwef{} involves showing that for any two agents $j$ and $\hat j$, agents $j$ and $\hat j$ receive equal shares of every bundle in expectation, due to a bijective mapping from the set of orders where $j$ picks before $\hat j$ and vice-versa.

\begin{restatable}{rmk}{rmkthmerp}\label{rmk:thm:erp}
	Under general partial preferences, \erp{} is not \upiva{} and \sdspa{}.
\end{restatable}

Given an assignment $P$ and a partial preference profile $R$, for any $\vx,\hat\vx \in \mathcal{D}$, $(\vx,\hat\vx)$ is an \emph{improvable tuple}, denoted by $\gcycle(P,R)$, if there exists an agent $j < n$ such that $\vx \succ_j \hat\vx$ and $p_{j,\hat\vx}>0$. We use $\gcycle(P)$ for short when the preferences are clear from the context. \citeay{Bogomolnaia01:New} show that an assignment $P$ is \sdopta{} if and only if the binary relation $\gcycle(P)$ has no cycle in single-type resources allocations, but the sufficient condition fails to hold for \mtaps{}. As Example~\ref{eg:sd} shows, assignment (\ref{m2}) is \sdnopta{}, but the set of the improvable tuples, $\{(1_F1_B,1_F2_B), (1_F1_B,$ $2_F1_B),$ $(1_F2_B,2_F1_B), (2_F2_B,2_F1_B)\}$, has no cycle.

\begin{restatable}{thm}{thmeps}\label{thm:eps}
	Under general partial preferences, $\eps{}$ satisfies  \sdopt{}, \sdwef{}, and \etoe{}.
\end{restatable}

To prove the \sdopt{} of \eps{}, we relax the cycle from bundles to items and find a sufficient condition for \sdopt{} in \mtaps{} under general partial preferences. For example in assignment (\ref{m2}), agent $1$ can extract $1_B$ from $2_F1_B$ to match $1_F$ and extract $2_B$ from $1_F2_B$ to match $2_F$ to improve her result. We show \eps{} satisfies \sdopt{} by proving \eps{} satisfies the sufficient condition. We find that one agent does not envy others at the first bundle of her topological sort since she consumes the bundle from the beginning to the end when the bundle is exhausted. Based on that, We can prove she does not envy others at any upper contour set under \eps{} by induction on bundles under the topological sorting.

\begin{restatable}{rmk}{rmkthmeps}\label{rmk:thm:eps}
	\eps{} is not \epopta{} since its output may not be \decoma{} when coming to multi-type resources. \eps{} is not \orfra{}, \sdefa{} and \upiva{} under general partial preferences.
\end{restatable}

PS is both \sdopta{} and \decoma{} in single-type resources allocations~\cite{Bogomolnaia01:New}, but not \decoma{} in \mtaps{}. Serial dictatorship~\cite{Mackin2016:Allocating,Hosseini19:Multiple} maintains both \sdopt{} and \decom{} even in \mtaps{} but perform badly in fairness. Thankfully, \egd{} not only satisfies \sdopt{} and \decom{} but the common requirement for fairness, \etoe{}.

\begin{restatable}{thm}{thmegd}\label{thm:egd}
	Under general partial preferences, $\egd{}$ satisfies \sdopt{}, \epopt{}, \etoe{}, and \decom{}.
\end{restatable}

When agents have different topological sorts with each other, \egd{} comes to serial dictatorship. Since the dictator only invite agents who have the same topological sort with her to share her bundle, the \sdopt{} of serial dictatorship maintains in \egd{}. Since agents with the same preference have the same topological sorts, \egd{} satisfies \etoe{}. \egd{} is \decoma{} since its output can be seen as an expected result of a special random priority algorithm. Decomposability and \sdopt{} induce \epopt{}~\cite{Bogomolnaia01:New}.
\begin{restatable}{rmk}{rmkthmegd}\label{rmk:thm:egd}
	\egd{} is not \orfra{}, \sdwefa{}, \upiva{} and \sdwspa{} even in single-type resources allocations under linear preferences.
\end{restatable}

PS is both \sdopta{} and \sdefa{} under linear preferences~\cite{Bogomolnaia01:New} but this is no longer true under general partial preferences, as the impossibility result in Proposition~\ref{prop:noboth} shows.

\begin{restatable}{prop}{propnoboth}\label{prop:noboth}
	No mechanism can satisfy both \sdopt{} and \sdef{} under general partial preferences.
\end{restatable}

\section{Properties under Acyclic CP-net Preferences}

\begin{restatable}{thm}{thmcperp}\label{thm:cperp}
	Given any CP-profile $R$, $\erp{}(R)$ is \sdspa{}.
\end{restatable}

\begin{proof}
    Let $P$ and $P'$ be the expected assignments of \erp{} given the CP-profile $R$ and $R'\in \mathcal{R}$ such that $R' = (\succ_j', \succ_{-j})$ for some $j\leq n$. For an arbitrary fixed priority order, in any agent $j$'s turn, the set of available bundles is the same under $R$ and $R'$. By Proposition~\ref{prop:acyclicCP}, the best available bundle is also the same and unique. If $j$ lies, she may get a smaller share of her best available bundle. Therefore, the result of a lie is stochastic dominated by the result of truthfulness. Since this is true for any fixed priority order, $P'$ is stochastic dominated by $P$.
\end{proof}

\begin{restatable}{prop}{propcperpuiv}\label{prop:cperpuiv}
	Given any CP-profile $R$, $\erp{}(R)$ is \upiva{} for any other CP-profile $R'$.
\end{restatable}

Thanks to Proposition~\ref{prop:acyclicCP}, given any upper invariant transformation at some $\vy\in\md$, if $\vy$ is available in the misreport agent's turn, she gets the same bundle despite her lie. Otherwise if $\vy$ is unavailable, the lie does not affect the assignment of $\vy$.

\begin{restatable}{thm}{thmcpepsef}\label{thm:cpepsef}
	Given any CP-profile $R$, $\eps{}(R)$ is \sdefa{} and \orfra{}.
\end{restatable}

\begin{proof}
     W.l.o.g we only prove the case between agent $1$ and agent $2$. Let $P = \eps{}(R)$. Let $\mathcal{D}_1 = \{\vx\in \mathcal{D}: P_{1,\vx} > 0\}$ and $ n_1 = | \mathcal{D}_1 |$. By Proposition~\ref{prop:acyclicCP}, we have an order over $\mathcal{D}_1$ such that $\vx_1 \succ_1 \vx_2 \succ_1 \cdots \succ_1 \vx_{n_1}$, $\vx_i \in \mathcal{D}_1$ for any $i \leq  n_1$. For agent $2$, we can define $\md_2$ and $n_2$ and have an order over $\md_2$, $\hat\vx_1 \succ_2 \hat\vx_2 \succ_2 \cdots \succ_2 \hat\vx_{n_2}$,  similarly.

    \emph{(1) \sdef{}.} We need to prove $P(1) \sd{1} P(2)$. For any $\vy \in \mathcal{D}$, let $\vx_i$ be the least favorable bundle of agent $1$ in $\ucs(\succ_1,\vy) \cap \mathcal{D}_1$. If $i = n_1$, $\sum_{\vx \in \ucs(\succ_1,\vy)} P_{1, \vx} = \sum_{k=1}^{n_1} P_{1, \vx_k} = 1 \geq \sum_{\vx \in \ucs(\succ_1,\vy)} P_{2, \vx}$. If $i < n_1$, $\vx_{i+1} \notin \ucs(\succ_1,\vy)$. When agent $1$ starts to consume $\vx_{i+1}$, $\hat\vy$ is unavailable for any $\hat\vy \in \ucs(\succ_1,\vy)$. Otherwise by Proposition~\ref{prop:acyclicCP}, we have $\vx_{i+1} \succ_1 \hat\vy$ or $\vx_{i+1} = \hat\vy$ both indicating $\vx_{i+1} \in  \ucs(\succ_1,\vy)$, a contradiction. Suppose $t(\vx_i)$ be the time when $\vx_i$ is exhausted and agent $1$ starts to consume $\vx_{i+1}$. Then $\sum_{\vx \in \ucs(\succ_1,\vy)}P_{1,\vx} = \sum_{k=1}^i P_{1, \vx_k} = t(\vx_i) \geq \sum_{\vx \in \ucs(\succ_1,\vy)}P_{2,\vx}$ for any $\vy \in \mathcal{D}$. Hence $P(1) \succ_1 P(2)$.

    \emph{(2) \orfr{}.} For any $\vy \in \md_1$, we need to prove $\sum_{\vx\in \ucs(\succ_1,\vy)}P_{1,\vx} \leq \sum_{\vx\in \ucs(\succ_2,\vy)}P_{2,\vx}$. For any $\vx_i \in {\md}_1$, suppose that when $\vx_i$ is exhausted, agent $2$ is consuming $\hat \vx_{\hat i} \in {\md}_2$ or $\hat \vx_{\hat i}$ is exhausted at the same time. If $\hat \vx_{\hat i} = \vx_i$, we have $\sum_{\vx\in\ucs(\succ_1, \vx_i)}P_{1,\vx} = \sum_{\vx\in\ucs(\succ_2, \vx_i)}P_{2,\vx}$. If $\hat \vx_{\hat i} \neq \vx_i$, since $\vx_i$ is available when agent $2$ starts to consume $\hat \vx_{\hat i}$, we have $\hat \vx_{k} \succ_2 \vx_i$ for any $k\leq \hat i$. Hence, $\sum_{\vx\in\ucs(\succ_1, \vx_i)}P_{1,\vx} = \sum_{k=1}^i P_{1,\vx_k} \leq \sum_{k=1}^{\hat i} P_{2,\hat \vx_k} \leq \sum_{\vx\in\ucs(\succ_2, \vx_i)}P_{2,\vx}$. Therefore, for any $\vy \in \md$ with $P_{1,\vy} > 0$, we have $\sum_{\vx\in\ucs(\succ_1, \vy)}P_{1,\vx} \leq \sum_{\vx\in\ucs(\succ_2, \vy)}P_{2,\vx}$.
\end{proof}
\citeay{Hashimoto14:Two} shows that \orfr{} characterizes PS when we come to single-type resources. However, this is not the case in \mtaps{}: for example, if two agents' preferences are both as (c) in Figure~\ref{fg:popg1} shows, the result that two agents both get 0.5 $1_F2_B$ and 0.5 $2_F1_B$ is \orfra{} but not the output of \eps{}.

\begin{restatable}{prop}{propcpepsuiv}\label{prop:cpepsuiv}
	Given any CP-profile $R$, $\eps{}(R)$ is \upiva{} for any other CP-profile $R'$.
\end{restatable}

Given any \upivtr{} at some $\vy\in\md$, we show the consumption processes are identical until $\vy$ is exhausted whether the misreport agent is truthful or lies, by induction on the bundles consumed by the misreport agent under $R$. Details are in the Appendix.

\begin{restatable}{rmk}{rmkpropcpepswsp}\label{rmk:prop:cpepswsp}
    \eps{} is not \sdwspa{} under acyclic CP-net profiles even when all agents have a shared dependency graph, and agents are only allowed to misreport their CPTs.
\end{restatable}

\begin{restatable}{prop}{propcpepswsp}\label{prop:cpepswsp}
	Given any independent CP-profile $R$, \eps{}(R) is \sdwspa{} for any profile $R'$.
\end{restatable}
Independent CP-nets are a special case of acyclic CP-nets where the dependency graph has no edges, and the CPT for each type represents preferences over items of that type without any conditional dependencies on the allocation of items of other types. Strict preferences over a single type of resources are a special case where $p=1$. Moreover, when $p=1$, \eps{} is equivalent to PS, and at any point during the execution of \eps{} agents consume their favorite remaining item until it is exhausted. The proof for ~\Cref{prop:cpepswsp} is along similar lines to the proof of \sdwsp{} of PS due to~\cite{Bogomolnaia01:New}.

\section{Conclusion and Future Work}
We proposed and studied \erp{} and \eps{} as extensions of RP and PS to \mtaps{}. We also proposed \egd{} that is both \sdopta{} and \decoma{}. For future work, we are interested in axiomatic characterization of \erp{},  \eps{} and \egd{}. More generally, designing desirable mechanism for \mtaps{} is still an interesting and challenging open question.

\section{Acknowledgments}
We are grateful to the anonymous reviewers for their helpful comments. LX acknowledges NSF \#1453542 and \#1716333 for support. YC acknowledges NSFC under Grants 61772035, 61751210, and 61932001, and the National Science and Technology Major Project for IND (investigational new drug) under Grant 2018ZX09201-014 for support. HW acknowledges NSFC under Grants 61572003 and 61972005, and the National Key R\&D Program under Grants 2018YFB1003904 and 2018YFC1314200 for support.

\bibliography{references}
\bibliographystyle{aaai}

\appendix
\section{Appendix}

\paragraph{Proof of Proposition~\ref{prop:acyclicCP}.}
\propacyclicCP*
\begin{proof}
    Given any acyclic CP-net $\succ$, suppose the dependency graph in the CP-net is $G$. As $G$ has no circle, there is $D_{i_1}\in T$ such that $D_{i_1}$ has no parent and there is $o_1$ in $D_{i_1}'$ such that the agent prefers $o_1$ than any other item in $D_{i_1}'$. Then there is $D_{i_2}\in T$ such that $Pa(D_{i_2})\subseteq \{D_{i_1}\}$. By $CPT(D_{i_2})$, we know the agent has a linear preference over $D_{i_2}$ given $o_1$. So there is $o_2 \in D_{i_2}'$ such that the agent prefers $o_2$ than any other item in $D_{i_2}'$ when she is holding $o_1$. The same procedure can be easily adapted to obtain $o_3$ in $D_{i_3}'$, $o_4$ in $D_{i_4}'$,..., $o_p$ in $D_{i_p}'$.

    We claim that $o_1o_2...o_p \succ \hat o_1\hat o_2...\hat o_p$ for any $\hat o_1\hat o_2...\hat o_p \in \mathcal{D}'$ and $\hat o_1\hat o_2...\hat o_p \neq o_1o_2...o_p$. If $o_1 \neq \hat o_1$, the agent prefers $o_1$ than $\hat o_1$ and we have $o_1\hat o_2...\hat o_p\succ \hat o_1\hat o_2...\hat o_p$, else we have $o_1\hat o_2...\hat o_p = \hat o_1\hat o_2...\hat o_p$. Since $Pa(D_{i_{k+1}})\subseteq\{D_{i_1},D_{i_2},...,D_{i_k}\}$ for any $k<p$, the agent prefers $o_{k+1}$ than $\hat o_{k+1}$ given $o_1,o_2,...,o_{k}$ or $o_{k+1} = \hat o_{k+1}$. Therefor, $o_1o_2...o_k o_{k+1}\hat o_{k+2}...\hat o_p \succ o_1o_2...o_k\hat o_{k+1}\hat o_{k+2}...\hat o_p$ or $o_1o_2...o_k o_{k+1}\hat o_{k+2}...\hat o_p =o_1o_2...o_k\hat o_{k+1}\hat o_{k+2}...\hat o_p$. Since $\succ$ is transitive and $o_1o_2...o_p \neq \hat o_1\hat o_2...\hat o_p$, we have $o_1o_2...o_p \succ \hat o_1\hat o_2...\hat o_p$. Since the selection of $o_k$ ($k\leq p$) only depends on the selections of its parents, $o_1o_2...o_p$ is unique.

\end{proof}

\paragraph{Proof of Proposition~\ref{prop:cost}.}
\propcost*
\begin{proof}
\emph{(1) \erp{}.} For a preference graph with $n^p$ nodes, we can get its topological sort in $O(n^p)$ via a queue and a vector to mark. Hence, the running time of line 1 in \erp{} is $O(n^{p+1})$. Line 1 run in $O(n)$. In line 4, \erp{} cost $O(n^p)$ to get $Ext$ for each agent. Hence, line 4 cost $O(n^{p+1})$ totally. Other lines run in constant. Hence, although the computation of the expected results of \erp{} is difficult, as a random algorithm, \erp{} run in $O(n^{p+1})$ time.

\emph{(2) \eps{}.} Line 1 cost $O(n^{p+1})$ in \eps{}. Line 2 run in constant. Since we can get all $Ext$s we need for all agents in all rounds by going through each topological sort only once. The total running time in line 4 is $O(n^{p+1})$. Since \eps{} remove at least one item in $M'$ each round, \eps{} run $np$ rounds at most to consume. In each round of consumption, \eps{} cost $O(np)$ in line 5.1, 5.2, 5.4 and $O(n)$ in line 5.3. Hence, \eps{} cost $O(n^2p^2)$ to consume totally. Since $p$ is a small number, the cost of \eps{} is still $O(n^{p+1})$.

\emph{(3) \egd{}.} Line 1 cost $O(n^{p+1})$ in \egd{}. Line 4 cost $O(n^p)$ in each round and $O(n^{p+1})$ totally. There are some previous work for getting $Group(j,\succ')$ in line 5. We should compute the segmentation of $N$ such that agents with the same topological sort is in and only in the same set. To do so, we let all agents be in a same set with $n^p$ rounds following. In each round $k \leq n^p$, for each set, visit the $k$-th bundles in the topological sorts of the agents in the set, divide agents in the set into several new sets such that the visited bundles of agents in the same new set is same. That can be done in linear time by using a map with $n^p$ size in each round. The final sets make up the segmentation. The previous work cost $O(n^{p+1})$ totally. After the previous work, line 5 can be done in $O(n)$. The total cost in line 6 is $O(np)$ since there are only $np$ items totally. Hence, the running time of \egd{} is $O(n^{p+1})$.

\end{proof}

\paragraph{Proof of  Theorem~\ref{thm:erp}.}
\thmerp*
\begin{proof}
    \emph{(1) \epopt{}.} Given any priority order of agents in \erp{}, suppose $P$ is the output. Assume for the sake of contradiction that there exists an assignment $Q \neq P \in \mathcal{P}$ such that $Q \sd{} P$. Let $j$ be the first agent in the set of agents who get different bundles between $P$ and $Q$. Then, any bundle she gets in $Q$ is available in her turn in \erp{}. Let $\vx$ be the bundle agent $j$ gets in $P$. Since $Q \sd{} P$, there is a bundle $\vy \in \md$ such that $\vy \succ_j \vx$, $Q_{j,\vy} > 0$ and $\vy$ is available in agent $j$'s turn. That is contradicted with that agent $j$ gets $\vx$ in her turn. Therefore, the result of \erp{} is \sdopta{} for any fixed priority order. The expected result of \erp{} is the result by randomizing the priority orders, so it is the linear combination of some \sdopta{} discrete assignments.

    \emph{(2) \sdwef{}.} W.l.o.g we only consider agent $1$ and agent $2$. Let assignment $P$ be the expected result of \erp{}. Suppose $P(2) \sd{1} P(1)$. We need to prove that $P(2) = P(1)$.

    Suppose the topological order of agent $1$'s preference in \erp{} is $\vx_1\succ' \vx_2\succ'...\succ'\vx_{n^p}$. Let $O$ be the set of any possible priority order, $O_1$ be the set of orders in which agent $1$ is ranked in the front of agent $2$ and $O_2$ be the set of orders in which agent $2$ is ranked in the front of agent $1$. It is obvious that $\{O_1,O_2\}$ is one division of $O$. Let $f:O_1\mapsto O_2$ such that for any $o_1\in O_1$, $f(o_1)$ is the order that exchange the positions of agent $1$ and agent $2$ in $o_1$. $f$ is a one-to-one mapping and a surjection from $O_1$ to $O_2$ obviously.

    We define the proposition $F(\vx)$ as follow. For any $o_1\in O_1$, $o_2\in O_2$ and $f(o_1)=o_2$, either case (i) or case (ii) happens:
    \begin{enumerate*}[label=(\roman*)]
    \item agent $1$ gets $\vx$ under $o_1$ or $o_2$ while agent $2$ gets $\vx$ under the other order; \item both agent $1$ and agent $2$ do not get $\vx$ under $o_1$ and $o_2$.
    \end{enumerate*}

    Then, we prove that $F(\vx_1)$ is true. We consider two cases.

    \begin{enumerate}
        \item For any $o_2\in O_2$, let $o_1\in O_1$ such that $f(o_1)=o_2$. If agent $2$ gets $\vx_1$ under $o_2$, $\vx_1$ is available for agent $1$ under $o_1$ and agent $1$ gets $\vx_1$ under $o_1$.

        \item For any $o_1\in O_1$, agent $1$ is front of agent $2$ and $\vx_1$ is the first bundle in topological order of agent $1$'s preference, so agent $2$ has no chance to get $\vx_1$ under $o_1$.
    \end{enumerate}
    Based on the two cases above, we can conclude that $p_{2,\vx_1}\leq p_{1,\vx_1}$. But $P(2) \sd{1} P(1)$ leads to $p_{2,\vx_1}\geq p_{1,\vx_1}$. So $p_{2,\vx_1}=p_{1,\vx_1}$ that induce that $F(\vx_1)$ is true in addition to the two points above.

    Next, supposing for any $i<k$ ($k\in\{2,3,...,n^p\}$), $p_{2,\vx_i}=p_{1,\vx_i}$ and $F(\vx_i)$ is true, we prove that $p_{2,\vx_k}=p_{1,\vx_k}$ and $F(\vx_k)$ is true. Consider the two cases again.

    \begin{enumerate}
        \item For any $o_2\in O_2$, there is $o_1\in O_1$ such that $f(o_1)=o_2$. If agent $2$ gets $\vx_k$ under $o_2$, $\vx_k$ is available for agent $1$ under $o_1$. Besides, it means that agent $2$ does not get $\vx_i$ for any $i<k$ under $o_2$. So the case (ii) in $F(\vx_i)$ happens, agent $1$ does not get $\vx_i$ under $o_1$. Agent $1$ choose $\vx_k$ in her turn under $o_1$.

        \item For any $o_1\in O_1$, there is $o_2\in O_2$ such that $f(o_1)=o_2$. If agent $2$ dose not get $\vx_k$ under $o_1$, it dose not influence the result that the expectation of $\vx_k$ agent $1$ gets is bigger than what agent $2$ gets. If agent $2$ gets $\vx_k$ under $o_1$, $\vx_k$ is available for agent $1$ under $o_2$. Agent $2$ does not get $\vx_i$ for any $i<k$ under $o_1$ and the case (ii) in $F(\vx_i)$ happens, that indicates Agent $1$ does not get $\vx_i$ for any $i<k$ under $o_2$. Hence, Agent $1$ gets $\vx_k$ under $o_2$.
    \end{enumerate}

    The two cases above imply that $p_{2,\vx_k}\leq p_{1,\vx_k}$.Since $P(2) \sd{1} P(1)$ and $p_{2,\vx_i}=p_{1,\vx_i}$ for any $i<k$ and $\ucs(\succ_1,\vx_k)\subseteq \{\vx_i: i \leq k\}$, we have $p_{2,\vx_k}\geq p_{1,\vx_k}$. Hence $p_{2,\vx_k}= p_{1,\vx_k}$ that induce that $F(\vx_k)$ is true in addition to the two points above.

    We have $P(2) = P(1)$ by induction.

    \emph{(3) \sdwsp{}.} W.l.o.g Suppose agent $1$ is the misreport agent. Given any partial preference $R$, let $R' = (\succ'_1, \succ_{-1})$ and $\succ'$ be the fixed topological order over $\succ_1$ in \erp{}. Let assignment $P$ and $P'$ be the expected results of \erp{} under $R$ and $R'$. Suppose $P' \succ^{sd}_1 P$. Assume for the sake of contradiction that $P'(1) \neq P(1)$. Let $\vx_0$ be the first bundle in $\succ'$ such that $P'_{1,\vx_0} \neq P_{1,\vx_0}$. We have $P'_{1,\vx_0} > P_{1,\vx_0}$ since $P' \succ^{sd}_1 P$. Under a priority order of agent, if agent $1$ gets $\vy\in\md$ such that $\vx_0 \succ' \vy$ under $R$, $\vx_0$ is unavailable under $R$ and $R'$. Hence, there is a priority order $\rhd_1$ such that agent $1$ gets $\vx_0$ under $R'$ and gets $\vx_1$ where $\vx_1 \succ' \vx_0$ under $R$, otherwise we have $P'_{1,\vx_1} \leq P_{1,\vx}$, a contradiction. If $\vx_1$ isn't the first bundle in $\succ'$, since $P'_{1,\vx_1} = P_{1,\vx}$, there is a priority order $\rhd_2$ such that agent $1$ gets $\vx_1$ under $R'$ and gets $\vx_2$ where $\vx_2 \succ' \vx_1$ under $R$. Continue the process until we find the first bundle $\vx_i$ in $\succ'$ and there is no such priority order for $\vx_i$. Hence, $P_{1,\vx_i} > P'_{1,\vx_i}$, a contradiction with $P' \succ^{sd}_1 P$. Therefore, $P'(1) = P(1)$.

    \emph{(4) \etoe{} and \decom{}.} That is obvious.

\end{proof}

\paragraph{Counterexample for Remark~\ref{rmk:thm:erp}.}
\rmkthmerp*
\begin{proof}
    Consider the case where $N = \{1,2\}$, $M = D_1 = \{1_F,2_F\}$, and $R=\{\succ_1,\succ_2\}$ where $\succ_1 = \emptyset$, $\succ_2 = \{1_F \succ_2 2_F\}$. In \erp{}, the fixed topological orders are $1_F \succ' 2_F$ for $\succ_1$ and $1_F \succ' 2_F$ for $\succ_2$. It is obvious that agents both get 0.5 $1_F$ and 0.5 $2_F$ under this situation. Let $R' = \{\succ'_1,\succ_2\}$ where $\succ'_1 = \{2_F \succ'_1 1_F\}$. Under $R'$, agent $1$ gets $2_F$ and agent $2$ gets $1_F$. Let $P$ and $P'$ be the expected output of \erp{}. We don't have $P \succ^{sd}_1 P'$. Hence, \erp{} is not \sdspa{} under general partial preferences. $\succ'_1$ is an \upivtr{} of $\succ_1$ at $2_F$ under $P$, but $P'_{1,2_F} \neq P_{1,2_F}$. Hence, \erp{} is not \upiva{} under general partial preferences.

\end{proof}

\paragraph{Proof of Theorem~\ref{thm:eps}.}
\thmeps*
\begin{proof}

    \emph{(1) \sdopt{}.} Given any partial preference $R$ and any assignment $P$ for an \mtap, a set $C \subseteq \gcycle(P)$ is a generalized cycle if it holds for every $o \in M$ that: if there exists an improvable tuple ${(\vx_1,\hat\vx_1)} \in C$ such that $o\in \vx_1$, then there exists a tuple ${(\vx_2,\hat\vx_2)} \in C$, such that $o\in \hat\vx_2$.

    \begin{restatable}{claim}{claimgc}\label{claim:gc}
    Given any partial preference profile $R$ and assignment $P$, $P$ is \sdopta{} if $P$ admits no generalized cycle at $R$.
    \end{restatable}

    \begin{claimproof}
        By inversions, we assume that the assignment $P$ admits no generalized cycle at $R$ and there exists another assignment $Q\neq P$ such that $Q \sd{\null} P$. We try to find a generalized cycle in $P$.

        Let $\hat N = \{j\in N: P(k) \neq Q(k)\}$ and $\hat N \neq \emptyset$ because $Q\neq P$. Let C be the set of all tuples ($\vx, \hat\vx$) such that for some $j \in \hat N$, $\vx \succ_j \hat\vx$ and $q_{j,\vx} > p_{j,\vx}$, $q_{j,\hat\vx} < p_{j,\hat\vx}$. Because $p_{j,\hat\vx} > q_{j,\hat\vx} \geq 0$, we have $C \subseteq \gcycle(P)$.

        For any $j \in \hat N$, there exists $\vx \in \mathcal{D}$ such that $q_{j,\vx} \neq p_{j,\vx}$. If we have $q_{j,\vx} > p_{j,\vx}$ for all $q_{j,\vx} \neq p_{j,\vx}$, we would have $1=\sum_{\vx\in\mathcal{D}}q_{j,\vx} > \sum_{\vx\in\mathcal{D}}p_{j,\vx}=1$ which is a contradiction. Hence there exists $\vx_0 \in \mathcal{D}$ such that $q_{j,\vx_0} < p_{j,\vx_0}$. We have $\sum_{\vx \in \ucs(\succ_j, \vx_0)} q_{j,\vx} \geq \sum_{\vx \in \ucs(\succ_j, \vx_0)} p_{j,\vx}$ because $Q \succ P$. Hence, there exists some $\vx_1 \in \mathcal{D}$ such that $\vx_1 \succ_j \vx_0$ and $q_{j,\vx_1} > p_{j,\vx_1}$, otherwise we induce a contradiction that $\sum_{\vx \in \ucs(\succ_k, \vx_0)} q_{j,\vx} < \sum_{\vx \in \ucs(\succ_k, \vx_0)} p_{j,\vx}$. Therefore, $(\vx_1,\vx_0) \in C$ and $C \neq \emptyset$.

        Suppose for the sake of contradiction that $C$ is not a generalized cycle. Then there exists an item $o$ such that for any $\vx \in \mathcal{D}$ containing $o$ and there is $\vx$ in some tuple in $C$, $\vx$ is always the left component of any improvable tuple involving $\vx$. Obviously, we have $q_{j,\vx} = p_{j,\vx}$ for any $j\notin \hat N$. Suppose there exists $j \in \hat N$ such that $q_{j,\vx} < p_{j,\vx}$. Since $\sum_{\hat\vx\in \ucs(\succ_j,\vx)} q_{j,\hat\vx} \geq \sum_{\hat\vx\in \ucs(\succ_j,\vx)} p_{j,\hat\vx}$, there exists some $\hat\vx \in \mathcal{D}$ such that $\hat\vx \succ_j \vx$ and $q_{j,\hat\vx} > p_{j,\hat\vx}$. Hence $(\hat\vx, \vx) \in C$ and $\vx$ is the right component which is a contradiction to $\vx$ is always the left component. Therefore, for any $j\in \hat N$, we have $q_{j,\vx} \geq p_{j, \vx}$. Since $\vx$ is the left component of some tuple in $C$, there exists $j_1\in N$, such that $q_{j_1,\vx} > p_{j_1,\vx}$. Therefore $1 = \sum_{j\in N, o\in\vx}q_{j,\vx} > \sum_{j\in N, o\in\vx}p_{j,\vx} = 1$, a contradiction. Therefore for any tuple $(\vx, \hat\vx) \in C$ and $o\in \vx$, there is a tuple $(\hat\vy, \vy)\in C$ such that $o\in\vy$. Hence $C$ is a generalized cycle.
    \end{claimproof}

    Suppose $P$ is the output of \eps{} given any partial preference profile $R$. We next prove $P$ admits no generalized cycle, and therefore is \sdopta{} by Lemma~\ref{claim:gc}.

    Suppose for the sake of contradiction that $P$ admits a generalized cycle $C$. Let $(\vx, \hat \vx)$ be an arbitrary improvable tuple in $C$. Let \emph{Seq} be the partial order on $M$ which reflects the time when items are exhausted. Formally, for any pair of items $o$ and $\hat o$ exhausted at times $t$ and $\hat t$ during the execution of ESP respectively, if $t \leq \hat t$, then $o$ \emph{Seq} $\hat o$.

    Let $\hat M$ be the set of items that are in some left components involved in C, and let $\hat o \in \hat M$ be one of top items according to \emph{Seq}. By definition of generalized cycles, there is an improvable tuple $(\vx, \hat\vx)\in C$ such that $\hat o \in \hat\vx$. Since $(\vx, \hat\vx) \in \gcycle(P)$, there exists an agent $j\in N$ such that $\vx \succ_j \hat\vx$ and $p_{j,\hat\vx} > 0$. Hence when agent $j$ starts to consume $\hat\vx$, $\vx$ is unavailable and there is an item $o\in \vx$ such that $o$ is unavailable when $j$ starts to consume $\hat\vx$. If we use $t(o)$ and $t(\hat o)$ to stand for the times when $o$ and $\hat o$ exhausted in \eps{} respectively, we have $t(o) \leq t(\hat o) - p_{j,\hat \vx} < t(\hat o)$. Besides, since $(\vx, \hat\vx) \in C$, we have $o\in \hat M$, a contradiction to that $\hat o$ is one of top items according to \emph{Seq} in $\hat M$.

    \emph{(2) \sdwef{}.} Suppose $P$ is the output of \eps{} given any partial preference profile $R$. W.l.o.g we only prove the case for agent $1$ and agent $2$, that is if $P(2) \sd{1} P(1)$, then $P(1) = P(2)$.

    Suppose $P(2) \sd{1} P(1)$. Suppose the topological order of agent $1$'s preference in \eps{} is $\vx_1 \succ' \vx_2 \succ' ...\succ' \vx_{n^p}$. There exists $0=t_0\leq t_1 \leq \cdots \leq t_{n^p} = 1$ such that for any $i \leq n^p$, if $p_{1,\vx_i} > 0$, $t_i> t_{i-1}$ and $t_i$ is the time when $\vx_i$ is exhausted, else $t_i = t_{i-1}$ and $\vx_i$ is unavailable at $t_i$.

    At the beginning, any bundle is available. Hence, $p_{1,\vx_1} > 0$ and agent $1$ starts to consume $\vx_1$ at $t_0 = 0$ till $\vx_1$ is exhausted. That indicates $p_{1,\vx_1} \geq p_{2,\vx_1}$. Since $P(2) \sd{1} P(1)$ and $\vx_1$ is the first bundle in the topological order over $\succ_1$, we have $p_{2,\vx_1} = \sum_{\vx \in \ucs(\succ_1,\vx_1)} p_{2,\vx} \geq \sum_{\vx \in \ucs(\succ_1,\vx_1)} p_{1,\vx} = p_{1,\vx_1}$. Therefore, $p_{1,\vx_1} = p_{2,\vx_1}$ and agent $1$ and agent $2$ are both consuming $\vx_1$ in the whole interval of $[t_0, t_1)$.

    Suppose for any $i < k$, $p_{1,\vx_i} = p_{2,\vx_i}$ and agent $1$ and agent $2$ are both consuming the same bundle at any time in $[t_{i-1}, t_i)$. We want to prove agent $1$ and agent $2$ are both consuming the same bundle at any time in $[t_{k-1}, t_k)$. If $t_k = t_{k-1}$, that is true for $[t_{k-1}, t_k)$. If $t_k > t_{k-1}$, agent $1$ is consuming $\vx_k$ in the whole interval of $[t_{k-1}, t_k)$ but never in $[t_0, t_{k-1})$. Hence, agent $2$ never consumes $\vx_k$ in $[t_0, t_{k-1})$. Since $\vx_k$ is exhausted at $t_k$, we have $p_{1,\vx_k} = t_k - t_{k-1} \geq p_{2,\vx_k}$. Since $P_2 \sd{1} P_1$, we have $\sum_{\vx \in \ucs(\succ_1,\vx_k)} p_{2,\vx} \geq \sum_{\vx \in \ucs(\succ_1,\vx_k)} p_{1,\vx}$. Since $\ucs(\succ_1,\vx_k) \subseteq \{\vx_1, \vx_2,..., \vx_{k-1}, \vx_k\}$ and $p_{1,\vx_i} = p_{2,\vx_i}$ for any $i < k$, we have $p_{2,\vx_k} \geq p_{1,\vx_k}$. Therefore, $p_{1,\vx_k} = p_{2,\vx_k}$ and agent $1$ and agent $2$ are both consuming $\vx_k$ in the whole interval of $[t_{k-1}, t_k)$.

    By induction, we have $p_{1,\vx_k} = p_{2,\vx_k}$ for any $k \leq n^p$. Hence $P(1) = P(2)$.

    \emph{(3) \etoe{}.} That is obvious.

\end{proof}

\paragraph{Counterexamples for Remark~\ref{rmk:thm:eps}.}
\rmkthmeps*
\begin{proof}
    \emph{(1) \eps{} is not \epopta{} coming to multi-type resources.} We consider a special case for partial preferences, CP-net preferences with the shared dependency graph as shown in Figure~\ref{fg:popg3}. We can induce that the preferences are $1_F2_B\succ_1 1_F1_B\succ_1 2_F1_B\succ_1 2_F2_B$ and $1_F1_B\succ_2 1_F2_B\succ_2 2_F2_B\succ_2 2_F1_B$. $\eps{}(R)$ is shown in assignment~\ref{m4} which cannot be realized obviously.

 	\begin{equation}\label{m4}
	\begin{tabular}{|l|c|c|c|c|}\hline
	&  $1_F1_B$   & $1_F2_B$  & $2_F1_B$    & $2_F2_B$  \\
	\hline
	\text{Agent 1} &  0  & 0.5   & 0.5     & 0    \\
	\hline
	\text{Agent 2} &  0.5    & 0   & 0     & 0.5    \\
	\hline
	\end{tabular}
	\end{equation}
	
    \begin{figure}[h]
        \centering
        \includegraphics[width=0.35\textwidth]{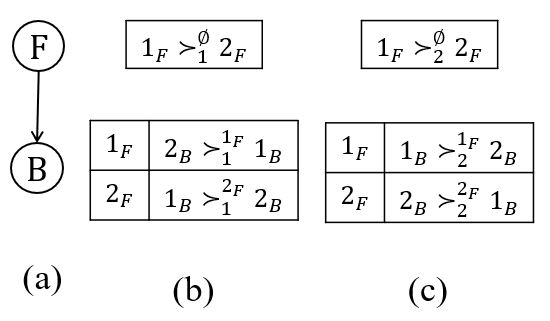}
        \caption{CP-net preferences with the shared dependency graph: (a) the dependency graph for agents; (b) CPTs for agent $1$; (c) CPTs for agent $2$.}
        \label{fg:popg3}
    \end{figure}

    \emph{(2) \eps{} is not \orfra{} and \sdefa{} under general partial preferences.} Consider the case where $N = \{1,2\}$, $M = D_1 = \{1_F,2_F\}$, and $R=\{\succ_1,\succ_2\}$ where $\succ_1 = \emptyset$, $\succ_2 = \{1_F \succ_2 2_F\}$. In \eps{}, the fixed topological orders are $2_F \succ' 1_F$ for $\succ_1$ and $1_F \succ' 2_F$ for $\succ_2$. It is obvious agent $1$ gets $2_F$ and agent $2$ gets $1_F$. Let $P = \eps{}(R)$. We have $P_{2,1_F} > 0$ and $\sum_{\vx \in \ucs(\succ_2, 1_F)}P_{2,\vx} > \sum_{\vx \in \ucs(\succ_1, 1_F)}P_{1,\vx}$. Hence, \eps{} is not \orfra{}. We don't have $P(1) \succ^{sd}_1 P(2)$. Hence, \eps{} is not \sdefa{}.

    \emph{(3) \eps{} is not \upiva{} under general partial preferences.} The counterexample of \upiv{} for \erp{} can be also an counterexample for \eps{}.

\end{proof}

\paragraph{Proof of Theorem~\ref{thm:egd}.}
\thmegd*

\begin{proof}

    Given any partial preference $R$, there are $n$ eating rounds in \egd{}. For each round $j \leq n$, let $\topb{j}{}$ be the first available bundle in agent $j$'s topological sort and $Group(j,\succ')$ be the set of agents who have the same topological sort $\succ'_j$. Then for each $j'\in Group(j,\succ')$, $P_{j',\topb{j}{}}\gets \frac{1}{|Group(j,\succ')|}$.

    \emph{(1) \sdopt{}.} Given any partial preference profile $R$, let $P = \egd{}(R)$. Assume for the sake of contradiction that there exists an assignment $Q \neq P \in \mathcal{P}$ such that $Q \sd{} P$.

    In round $j = 1$, for each $j'\in Group(j,\succ')$, since $\sum_{\vx\in U(\succ_{j'},\topb{1}{})}P_{j',\vx} \leq \sum_{\vx\in U(\succ_{j'},\topb{1}{})}Q_{j',\vx}$ and $U(\succ_{j'},\topb{1}{}) = \{\topb{1}{}\}$, we have $P_{j',\topb{1}{}} \leq Q_{j',\topb{1}{}}$. Since $\sum_{j'\in Group(1,\succ')} P_{j',\topb{1}{}} = 1 \geq \sum_{j'\in Group(1,\succ')} Q_{j',\topb{1}{}}$, we have $P_{\hat{j},\topb{1}{}} = Q_{\hat{j},\topb{1}{}}$ for each $\hat{j}\leq n$. Suppose for $j < k$, $P_{\hat{j},\topb{j}{}} = Q_{\hat{j},\topb{j}{}}$ for each $\hat{j}\leq n$. When $j = k$, for any $j'\in Group(k,\succ')$ and any $\vx \neq \topb{k}{} \in U(\succ_{j'},\topb{k}{})$, $\vx$ is unavailable. For any $\vx \neq \topb{k}{} \in U(\succ_{j'},\topb{k}{})$,  $Q_{j',\vx} \leq P_{j',\vx}$ since $Q_{j',\vx} > P_{j',\vx}$ implies $\vx$ is available in round $k$. Hence, $Q_{j',\vx} = P_{j',\vx}$, otherwise, $Q_{j',\vx} < P_{j',\vx}$ would cause a contradiction with $Q \sd{} P$. By $\sum_{\vx\in U(\succ_{j'},\topb{k}{})}P_{j',\vx} \leq \sum_{\vx\in U(\succ_{j'},\topb{k}{})}Q_{j',\vx}$, we have $P_{j',\topb{k}{}} \leq Q_{j',\topb{k}{}}$ for each $j'\in Group(k,\succ')$. Since $\sum_{j'\in Group(k,\succ')} P_{j',\topb{k}{}} = 1 \geq \sum_{j'\in Group(k,\succ')} Q_{j',\topb{k}{}}$, we have $P_{\hat{j},\topb{k}{}} = Q_{\hat{j},\topb{k}{}}$ for each $\hat{j}\leq n$. By induction on $k$, we have $P_{\hat{j},\topb{j}{}} = Q_{\hat{j},\topb{j}{}}$ for each $\hat{j},j\leq n$. Hence, $P = Q$, a contradiction.

    Therefore, $P$ is \sdopta{}.

    \emph{(2) \etoe{}.} That is obvious.

    \emph{(3) \decom{}.} Let $S = \{ Group(j, \succ') : j \leq n\}$ which is an segmentation of $N$. For $l \leq |S|$, $s_l \in S$ and $m \leq |s_l|$, sort agents in $s_l$ in ascending order and let $s_{l,m}$ be the $m$-th agent. Let $k$ be the least common multiple of $\{|s|: s \in S\}$. We forge a set of $k$ priority orders denoted $O$. For $u \leq k$ and $v \leq n$, let $O_{u,v}$ be the $v$-th agent in the $u$-th priority order in $O$. For each $u \leq k$, $l \leq |S|$, and $m \leq |s_l|$, let $w = ((m+u-2)\bmod |s_l|)+1$ and $O_{u,(s_{l,m})} = s_{l,w}$. Consider the special random priority mechanism that random in $O$ with $\frac{1}{k}$ possibility for each priority order. Since agents in the same group have the same partial preference, it is obvious the expected result of the mechanism is that for each $j'\in Group(j,\succ')$, $P_{j',\topb{j}{}}\gets \frac{1}{|Group(j,\succ')|}$. That is same as $\egd{}(R)$. Since the expected result of random priority mechanism is \decoma{}, $\egd{}(R)$ is \decoma{}.

    \emph{(4) \epopt{}.} Decomposability and \sdopt{} induce \epopt{}~\cite{Bogomolnaia01:New}.

\end{proof}

\paragraph{Counterexamples for Remark~\ref{rmk:thm:egd}.}
\rmkthmegd*
\begin{proof}
Consider the case where $N = \{1,2,3\}$, $M = D_1 = \{1_F, 2_F, 3_F\}$, and $R = \{\succ_1, \succ_2, \succ_3\}$ where $\succ_1 = \{1_F \succ_1 2_F, 1_F \succ_1 3_F, 2_F \succ_1 3_F\}$, $\succ_2 = \{1_F \succ_2 3_F, 1_F \succ_2 2_F, 3_F \succ_2 2_F\}$, and $\succ_3 = \{3_F \succ_3 1_F, 3_F \succ_3 2_F, 1_F \succ_3 2_F\}$. It is obvious that agent $1$ gets $1_F$, agent $2$ gets $3_F$ and agent $3$ gets $2_F$ in \egd{}. Let $P = \egd{}(R)$. Since $P_{1,1_F} = 1 > P_{2, 1_F}$ and $1_F$ is better than $3_F$ for agent $2$, \egd{} is not \orfra{} and \sdwefa{}. Suppose agent $3$ misreport her preference as same as agent $1$, which is an \upivtr{} of $\succ_3$ at $2_F$ under $P$. Agent $1$ and agent $3$ share their bundles. Agent $3$ gets 0.5 $1_F$ and 0.5 $2_F$ which is better than only $2_F$ for him. Hence, \egd{} is not \upiva{} and \sdwspa{}.

\end{proof}

\paragraph{Proof of Proposition~\ref{prop:noboth}.}
\propnoboth*
\begin{proof}
    We just consider the case of the allocation of single-type items and it is easy to example in similar manner when $p\geq$. Consider the case where $N = \{1,2,3\}$, $M = D_1 = \{1_F,2_F,3_F\}$, and $R=(\succ_j)_{j\leq 3}$ where $\succ_1 = \{1_F\succ_1 2_F, 1_F\succ_1 3_F, 2_F\succ_1 3_F\}$, $\succ_2 = \{3_F\succ_2 2_F, 3_F\succ_2 1_F, 2_F\succ_2 1_F\}$, and $\succ_3 = {\emptyset}$. It is easy to conclude that assignment (\ref{nobothm1}) is the only assignment that agent $3$ does not envy others under the concept of \sdef{}. However, it is dominated by assignment (\ref{nobothm2}).

	\begin{equation}\label{nobothm1}
	\begin{tabular}{|l|c|c|c|}\hline
	 &  $1_F$   & $2_F$  & $3_F$    \\
	\hline
	\rule{0pt}{11pt}\text{Agent 1} &  $\frac{1}{3}$  & $\frac{1}{3}$   & $\frac{1}{3}$  \\
	\hline
	\rule{0pt}{11pt}\text{Agent 2} &  $\frac{1}{3}$  & $\frac{1}{3}$   & $\frac{1}{3}$  \\
	\hline
	\rule{0pt}{11pt}\text{Agent 3} &  $\frac{1}{3}$  & $\frac{1}{3}$   & $\frac{1}{3}$  \\
	\hline
	\end{tabular}
	\end{equation}
	\begin{equation}\label{nobothm2}
	\begin{tabular}{|l|c|c|c|}\hline
	&  $1_F$   & $2_F$  & $3_F$    \\
	\hline
	\rule{0pt}{11pt}\text{Agent 1} &  $\frac{2}{3}$  & $\frac{1}{3}$   & $0$  \\
	\hline
	\rule{0pt}{11pt}\text{Agent 2} &  $0$  & $\frac{1}{3}$   & $\frac{2}{3}$  \\
	\hline
	\rule{0pt}{11pt}\text{Agent 3} &  $\frac{1}{3}$  & $\frac{1}{3}$   & $\frac{1}{3}$ \\
	\hline
	\end{tabular}
	\end{equation}
	
\end{proof}

\paragraph{Proof of Proposition~\ref{prop:cperpuiv}.}
\propcperpuiv*
\begin{proof}
    W.l.o.g we only prove the cases when agent $1$ lies. Given any CP-profile $R = (\succ_j)_{j\leq n}$, let $P$ be the expected output of $\erp{}$ under $R$. For any CP-profile $R' = (\succ'_1,\succ_{-1})$ where $\succ'_1$ is an upper invariant transformation of $\succ_1$ at $\vy \in \md$, let $P'$ be the expected output of $\erp{}$ under $R'$. Next, we prove that $P_{j,\vy} = P'_{j,\vy}$ for any $j \leq n$ by showing that given any priority order $\rhd$ over agents, any agent gets $\vy$ under $R'$ if and only if she gets $\vy$ under $R$.

    Since the preferences of other agents keep same, given $\rhd$, the sets of available bundles are same in agent $1$'s turn under $R$ and $R'$. Suppose $\md' \subseteq \md$ is the set of all available bundles in agent $1$'s turn in $\erp{}$ under $R$ and $R'$. Suppose agent $1$ gets $\vx \in \md'$ under $R$ and $\vx' \in \md'$ under $R'$. By Proposition~\ref{prop:acyclicCP}, $\vx \succ_1 \vz$ for any $\vz \neq \vx \in \md'$ and $\vx' \succ'_1 \vz'$ for any $\vz' \neq \vx' \in \md'$. Given $\rhd$, we divide all situations to two cases, $\vy \in \md'$ and $\vy \notin \md'$.

    Case 1, $\vy \in \md'$. We have $\vx \in \ucs{}(\succ_1,\vy)$ and $\vx' \in \ucs{}(\succ'_1,\vy)$. Since $P_{1,\vx}>0$, we have $\vx \in \ucs{}(\succ'_1,\vy)$ by the definition of $\upivtr{}$. If $\vx \neq \vx'$, we have $\vx \succ_1 \vx'$ and $\vx' \succ'_1 \vx$ which is a contradiction to the definition of \upivtr{} where $\succ'_1 \mid_{\ucs{}(\succ'_1,\vy)} = \succ_1 \mid_{\ucs{}(\succ'_1,\vy)}$. Hence, $\vx = \vx'$ and agent $1$ gets $\vx$ under $R$ and $R'$. Therefore, all agents get the same bundles under $R$ and $R'$ in this case.

    Case 2, $\vy \notin \md'$. Agents after agent $1$ in $\rhd$ and agent $1$ get no $\vy$ both under $R$ and $R'$. Besides, agents before agent $1$ in $\rhd$ get the same bundles under $R$ and $R'$. Hence, one agent gets $\vy$ under $R'$ if and only if she gets $\vy$ under $R$ in this case.

    To sum up the two cases, given any priority order $\rhd$ over agents, any agent gets $\vy$ under $R'$ if and only if she gets $\vy$ under $R$. Therefore, $P_{j,\vy} = P'_{j,\vy}$ for any $j \leq n$.

\end{proof}

\paragraph{Proof of Proposition~\ref{prop:cpepsuiv}.}
\propcpepsuiv*
\begin{proof}
    W.l.o.g we only prove the cases when agent $1$ lies. Given any CP-profile $R = (\succ_j)_{j\leq n}$ and any CP-profile $R' = (\succ'_1,\succ_{-1})$ where $\succ'_1$ is an upper invariant transformation of $\succ_1$ at $\vy \in \md$, let $P = \eps{}(R)$ and $P' = \eps{}(R')$.

    Let $\hat{\md} = \{\vx\in \md: P_{1,\vx} > 0\}$ and $\hat n = | \hat{\md} |$. By Proposition~\ref{prop:acyclicCP}, we have an order over $\hat{\md}$ such that $\vx_1 \succ_1 \vx_2 \succ_1 \cdots \succ_1 \vx_{\hat n}$, $\vx_i \in \hat{\md}$ for any $i \leq \hat n$. There is $\vx_{i_0}$ such that under $R$, when agent $1$ starts to consume $\vx_{i_0}$, $\vy$ is available and after $\vx_{i_0}$ is exhausted, $\vy$ is unavailable. Suppose $t_i$ is the time when $\vx_i$ is exhausted for $i\leq \hat n$ under $R$. Let $t_0 = 0$. Next, we prove that the consumption processes of all agents are same before $t_{i_0}$ under $R$ and $R'$ by induction.

    Suppose before $t_k$ ($k < i_0$), the consumption processes of all agents are same under $R$ and $R'$ (That is true for $k = 0$). Hence, the remaining resources are same under $R$ and $R'$ at $t_k$. If agent $1$ turns to consume $\vx_{k+1}$ under $R'$ at $t_k$, the consumption processes are same between $t_k$ and $t_{k+1}$ because the preferences of the other agents are same under $R$ and $R'$. Next, we prove that $\vx_{k+1}$ is the exact bundle agent $1$ turns to at $t_k$. Suppose $\md' \subseteq \md$ is the set of all available bundle at $t_k$ under $R$ and $R'$. By Proposition~\ref{prop:acyclicCP}, $\vx_{k+1} \succ_1 \hat \vx$ for any $\hat \vx \neq \vx_{k+1} \in \md'$. As the same reason, there is $\vx \in \md'$ such that $\vx \succ'_1 \hat \vx$ for any $\hat \vx \neq \vx \in \md'$. Since $k < i_0$, $\vy$ is available at $t_k$. Hence, $\vx_{k+1} \in \ucs{}(\succ_1, \vy)$ and $\vx \in \ucs{}(\succ'_1, \vy)$. Since $P_{1,\vx_{k+1}}>0$, we have $\vx_{k+1} \in \ucs{}(\succ'_1, \vy)$ by the definition of \upivtr{}. If $\vx \neq \vx_{k+1}$, we have $\vx \succ'_1 \vx_{k+1}$ and $\vx_{k+1} \succ_1 \vx$ which is a contradiction to the definition of \upivtr{} where $\succ'_1 \mid_{\ucs{}(\succ'_1,\vy)} = \succ_1 \mid_{\ucs{}(\succ'_1,\vy)}$. Hence, $\vx = \vx_{k+1}$, agent $1$ turns to $\vx_{k+1}$ at $t_k$ under $R'$ and the consumption processes are same between $t_k$ and $t_{k+1}$ under $R$ and $R'$.

    By induction, the consumption processes of all agents are same before $t_{i_0}$ under $R$ and $R'$. Hence, the consumption processes are same until $\vy$ is exhausted. Hence, $P_{j,\vy} = P'_{j,\vy}$ for any $j \leq n$.

\end{proof}

\begin{figure*}[ht]
    \centering
    \includegraphics[width=0.9\textwidth]{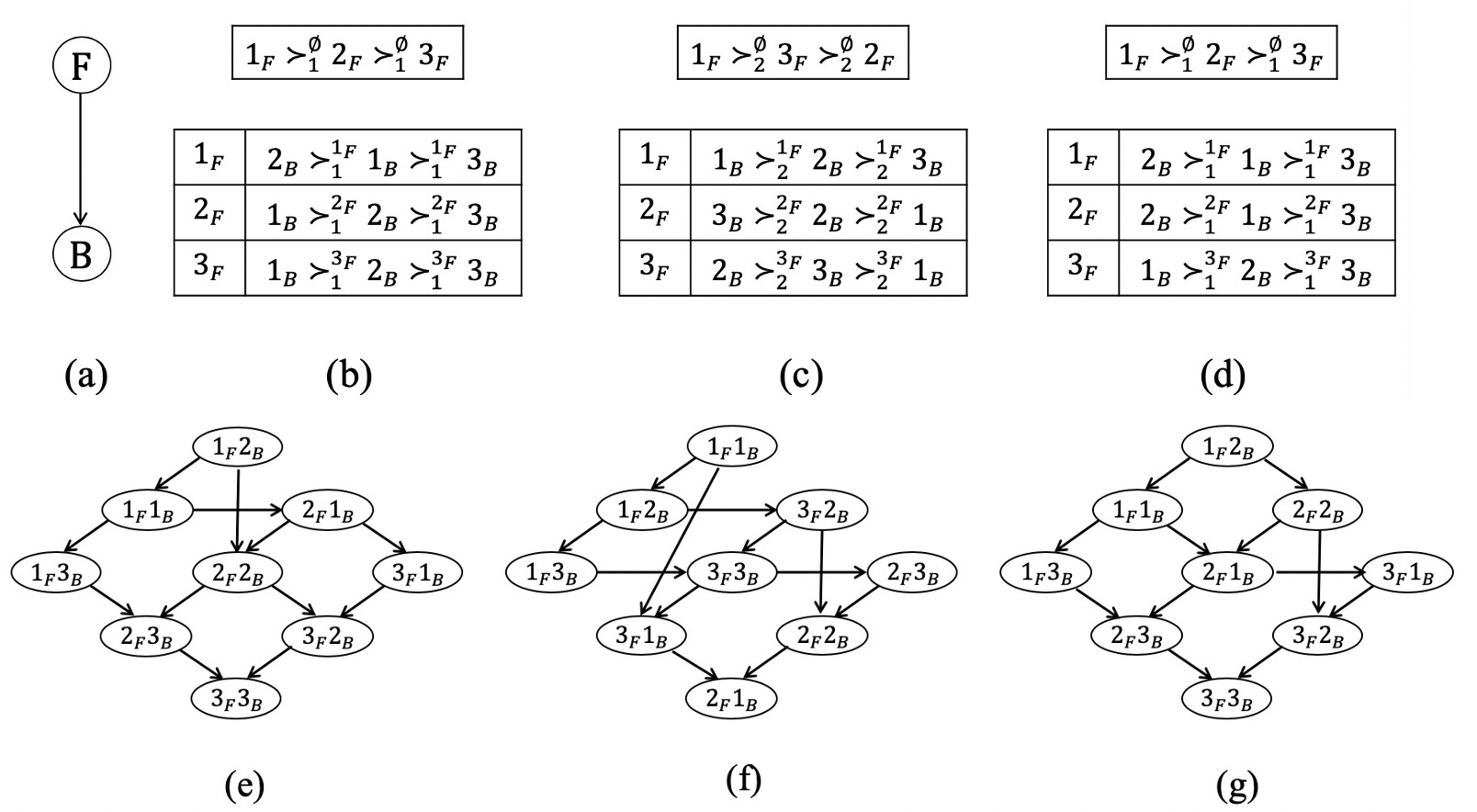}
    \caption{CP-net preferences with the shared dependency graph: (a) the shared dependency graph for agents; (b) CPTs for agent $1$; (c) CPTs for agent $2$ and $3$; (d) The misreported CPTs; (e) The corresponding preference graph of (b); (f) The corresponding preference graph of (c); (g) The corresponding preference graph of (d).}
    \label{fg:popg6}
\end{figure*}

\paragraph{Counterexample for Remark~\ref{rmk:prop:cpepswsp}.}
\rmkpropcpepswsp*
\begin{proof}
    Consider the \mtap{} where $N = \{1,2,3\}$, $M = \{1_F,2_F,3_F,1_B,2_B,3_B\}$, and let agents' preferences be the same as in Figure~\ref{fg:popg6}, where agents have a shared dependency graph shown in Figure~\ref{fg:popg6}(a). The CPTs of agent $1$ are shown as Figure~\ref{fg:popg6}(b) while agent $2$ and agent $3$ have the same CPTs as Figure~\ref{fg:popg6}(c) shows. Suppose that agent $1$ misreports her CPTs as Figure~\ref{fg:popg6}(d). Then, Figures~\ref{fg:popg6}(e),~\ref{fg:popg6}(f), and~\ref{fg:popg6}(g) are the preference graphs corresponding to the CPTs in Figures~\ref{fg:popg6}(b),~\ref{fg:popg6}(c), and~\ref{fg:popg6}(d) respectively with the shared dependency graph in Figure~\ref{fg:popg6}(a). When agent $1$ reports her CPTs as Figure~\ref{fg:popg6}(b), she gets $\frac{1}{3}1_F2_B$, $\frac{1}{3}2_F1_B$ and $\frac{1}{3}2_F3_B$ under \eps{}. But if agent $1$ misreports her CPTs as Figure~\ref{fg:popg6}(d), she gets $\frac{1}{3}1_F2_B$, $\frac{1}{3}2_F1_B$, $\frac{2}{9}2_F2_B$ and $\frac{1}{9}2_F3_B$ which is a better result for her under the notion of stochastic dominance. Hence, \eps{} is not \sdwspa{} even under this restriction.

\end{proof}

\paragraph{Proof of Proposition~\ref{prop:cpepswsp}.}
\propcpepswsp*
\begin{proof}
    Before the main body of the proof, we list three claims about independent CP-nets for convenience. The proofs of Claim~\ref{cl:sdincpnet} and~\ref{cl:orderinesp} are based on Claim~\ref{cl:strongtransitive}.

    \begin{restatable}{claim}{clstrongtransitive}\label{cl:strongtransitive}
        Given any independent CP-net $\succ$, for any $D_i \in T$, $o\in D_i$ and $\hat o \in D_i$, if there are some $\vy\in \Car{-\{D_i\}}, \hat\vy \in \Car{-\{D_i\}}$ such that $(o,\vy)\succ (\hat o,\hat\vy)$, then we have $o \in \ucs(\succ,\hat o)$.
    \end{restatable}

    Given any independent CP-net $\succ$, for any $D_i \in T$, $o \in D_i$, $\hat o\in D_i$, suppose there are some $\vy\in \Car{-\{D_i\}}, \hat\vy \in \Car{-\{D_i\}}$ such that $(o,\vy)\succ (\hat o,\hat\vy)$. Then there exists one path from $(o,\vy)$ to $(\hat o,\hat\vy)$ in the preference graph of $\succ$. Suppose the path is $(o_1,\vy_1) \succ (o_2,\vy_2) \succ \cdots \succ (o_t,\vy_t)$ such that $o_1 = o$, $\vy_1 = \vy$, $o_t = \hat o$, $\vy_t = \hat\vy$, and only one pair of items is different between $(o_k,\vy_k)$ and $(o_{k+1},\vy_{k+1})$ for any $k < t$. For any $\vz\in \Car{-\{D_i\}}$, we try to find a path from $(o,\vz)$ to $(\hat o,\vz)$ by referring the path from $(o,\vy)$ to $(\hat o,\hat\vy)$. Initialize $E$ as $\emptyset$. For any $k < t$, suppose the different items between $(o_k,\vy_k)$ and $(o_{k+1},\vy_{k+1})$ belong to $D_{i_k} \in T$. If $D_{i_k} = D_i$, we have $\vy_k=\vy_{k+1}$ which implies $(o_k,\vz) \succ (o_{k+1},\vz)$. Add the edge between $(o_k,\vz)$ and $(o_{k+1},\vz)$ to $E$. It is easy to see that the edges in $E$ together make up a path from $(o,\vz)$ to $(\hat o,\vz)$. Therefore $(o, \vz) \succ (\hat o,\vz)$ for any $\vz\in \Car{-\{D_i\}}$. Hence, $o \in \ucs(\succ,\hat o)$.

    \begin{restatable}{claim}{clsdincpnet}\label{cl:sdincpnet}
        Given any independent CP-profile $R$ and two assignments $Q$ and $P$, if $Q \succ^{sd}_j P$, for any $D_i \in T$, $\hat o \in D_i$, we have
        \[\sum_{o\in \ucs(\succ_j,\hat o)} Q_{j,o} \geq \sum_{o \in \ucs(\succ_j,\hat o)} P_{j,o}.\]
    \end{restatable}

    Given any independent CP-profile $R$ and two assignments $Q$ and $P$, suppose $Q \succ^{sd}_j P$. For any $\vx \in \md$ and $D \in T$, we define function $f_D$ as $f_D(\vx) = o$ such that $o\in \vx$ and $o\in D$. For any $D_i \in T$ and $\hat o \in D_i$, there is $\vx_0 \in \md$ such that $f_{D_i}(\vx_0) = \hat o$ and for any $D \neq D_i \in T$, $o \in D$, we have $f_D(\vx_0) = o$ or $o \succ_j f_D(\vx_0)$. For any $o\in \ucs(\succ_j, \hat o)$, $\vy \in \Car{-\{D_i\}}$, we have $(o,\vy) \succ_j \vx_0$. Hence,
    \[
    \sum_{o \in \ucs(\succ_j,\hat o)} Q_{j,o} = \sum_{o\in \ucs(\succ_j,\hat o)}\sum_{\vy\in \Car{-\{D_i\}}}Q_{j,(o,\vy)} \leq \sum_{\vx\in \ucs(\succ_j,\vx_0)} Q_{j,\vx}.
    \]
    For any $(o, \vy) \in \md$ such that $(o, \vy) \succ_j \vx_0$, by Claim \ref{cl:strongtransitive}, we have $o\in \ucs(\succ_j,\hat o)$. Hence,
    \[
    \sum_{o \in \ucs(\succ_j,\hat o)} Q_{j,o} = \sum_{o\in \ucs(\succ_j,\hat o)}\sum_{\vy\in \Car{-\{D_i\}}}Q_{j,(o,\vy)} \geq \sum_{\vx\in \ucs(\succ_j,\vx_0)} Q_{j,\vx}.
    \]
    Therefore,
    $\sum_{o \in \ucs(\succ_j,\hat o)} Q_{j,o} = \sum_{\vx\in \ucs(\succ_j,\vx_0)} Q_{j,\vx}$.
    In like manner, we have
    $\sum_{o \in \ucs(\succ_j,\hat o)} P_{j,o} = \sum_{\vx\in \ucs(\succ_j,\vx_0)} P_{j,\vx}$.
    Since $Q \succ^{sd}_j P$, $\sum_{\vx\in \ucs(\succ_j,\vx_0)} Q_{j,\vx} \geq \sum_{\vx\in \ucs(\succ_j,\vx_0)} P_{j,\vx}$. Hence we have $\sum_{o\in \ucs(\succ_j,\hat o)} Q_{j,o} \geq \sum_{o\in \ucs(\succ_j,\hat o)} P_{j,o}$.

    \begin{restatable}{claim}{clorderinesp}\label{cl:orderinesp}
    Under \eps{}, given any independent CP-profile $R$, for any $D_i\in T$ and $j \leq n$, there exists $o_1\succ_j o_2 \succ_j \cdots \succ_j o_l$ and $0= t_0 < t_1 < \cdots < t_l = 1$ such that $o_k \in D_i$, $t_k$ is the time when $o_k$ is exhausted and agent $j$ consumes $o_k$ in the time interval $[t_{k-1}, t_k]$ for $k \leq l$.
    \end{restatable}

    Given $D_i \in T$ and an independent CP-net $\succ_j$ which is the preference of agent $j \in N$. By Proposition~\ref{prop:acyclicCP}, at the beginning of \eps{}, agent $j$ starts to consume her unique favorite bundle $\vy_1 \in \md$. Let $o_1 \in D_i$ such that $o_1 \in \vy_1$. Suppose agent $j$ turns to consume $\hat \vy$ such that $o_1 \notin \hat \vy$ before $o_1$ is exhausted. Let $\hat \vy | o_1$ be the bundle which changes the item of $D_i$ in $\hat \vy$ to $o_1$. Then $\hat \vy | o_1$ is available when agent $j$ turns to $\hat \vy$. By $\vy_1 \succ_j \hat \vy$ and Claim~\ref{cl:strongtransitive}, we have $\hat \vy | o_1 \succ_j \hat \vy$, a contradiction. Hence, agent $j$ consumes $o_1$ in $[0,t_1)$ where $o_1$ is exhausted at $t_1$. Suppose agent $j$ turns to consume $\vy_2$ at $t_1$. Let $o_2 \in D_i$ such that $o_2 \in \vy_2$. In the same manner, we have agent $j$ consumes $o_2$ in $[t_1,t_2)$ where $o_2$ is exhausted at $t_2$. Continue the process. Finally, we find orders, $o_1\succ_j o_2 \succ_j \cdots \succ_j o_l$ and $0= t_0 < t_1 < \cdots < t_l = 1$, such that $o_k \in D_i$, $t_k$ is the time when $o_k$ is exhausted and agent $j$ consumes $o_k$ in the time interval $[t_{k-1}, t_k]$ for $k \leq l$.

    After the proofs of the three claims, we begin the main body of the proof. W.l.o.g we prove that it is not beneficial for agent $1$ to misreport. Let $R$ be an arbitrary independent CP-profile. Suppose for the sake of contradiction, that $R'=(\succ'_1,\succ_{-1})$ is a profile where agent $1$ misreports her preference as any partial preference relation $\succ'_1$, such that $\text{\eps{}}(R')\succ_1^{sd} \text{\eps{}}(R)$. Let $P = \text{\eps{}}(R)$ and $P' = \text{\eps{}}(R')$. We prove that $P' \succ_1^{sd} P$ implies that $P(1) = P'(1)$.

    For the sake of convenience, we define the following quantities to track the execution of $\text{\eps{}}(R)$. At any time $t$ during the execution of $\text{\eps{}}(R)$, let $S(j,t,D_i)$ be the item $o\in D_i$ that agent $j$ is consuming at $t$, and for any item $o\in M$, let $N(o,t)$ be the set of agents who are consuming an item $o$ at $t$ and $n(o,t) = |N(o,t)|$. For any $o \in M$, let $\tau(o)$ be the time when $o$ is exhausted during the execution of $\text{\eps{}}(R)$. Then, we have $\tau(o) = sup\{t:n(o,t) \geq 1\}$. For the execution of $\text{\eps{}}(R')$, we define $S'(j,t,D_i)$, $N'(a,t)$, $n'(a,t)$ and $\tau'(o)$ in a similar way.

    Our proof involves showing that $S(j,t,D_i) = S'(j,t,D_i)$ for any $j\leq n$, $t\in [1,0)$ and $D_i \in T$, which is a stronger condition than $P_1 = P_1'$.

    For any $D_i \in T$, let $O_{i}=\{o_1,\cdots,o_l\}$ be the set of items that agent $1$ gets a positive fractional share of in $\text{\eps{}}(R)$, i.e. for each item $o\in O_i$, $P_{1,o}>0$, and for each $\hat o \in D_i\setminus O_i$, $P_{1.\hat o}=0$. W.l.o.g. let $o_1\succ_1\cdots\succ_1 o_l$. Then, by Claim~\ref{cl:orderinesp}, there are time steps $0 = t_0 < t_1 < \cdots < t_l = 1$, such that for each $k\le l$, agent $1$ only consumes $o_k$ from $D_i$ during the whole interval $[t_{k-1}, t_k)$.

    Consider any type $i\le p$. We prove by induction that $S(j,t,D_i) = S'(j,t,D_i)$ for any $j\leq n$, $t\in [1,0)$. Since the proof for the base case where $S(j,t,D_i) = S'(j,t,D_i)$ for the interval $[t_0,t_1)$ is similar with the inductive steps, we only prove the induction hypothesis that for any $k\le l$, $S(j,t,D_i) = S'(j,t,D_i) = o_k$ for the interval $[t_{k-1}, t_k)$, for every $j\le n$.


    Assume that the induction hypothesis is true for each $m<k$. We claim that there exists no $o \in D_i$ such that $o \succ_1 o_k$, $o \notin O_i$, and $P'_{1,o} > 0$. Otherwise, since $S(j,t,D_i) = S'(j,t,D_i)$ for any $j \leq n$, $t \in [0, t_{k-1})$, $o$ is available at $t_{k-1}$, and $o \succ_1 o_k$. However, agent $1$ starts to consume $o_k$ at $t_{k-1}$, a contradiction to our assumption that $P=\text{\eps{}}(R)$. Therefore, by $P' \sd{} P$ and Claim~\ref{cl:sdincpnet}, we have $\sum_{m = 1}^k P_{1,o_m}' = \sum_{o \in \ucs(\succ_1, o_k)} P_{1,o}' \geq \sum_{o \in \ucs(\succ_1, o_k)} P_{1,o} = \sum_{m = 1}^k P_{1,o_m}$. By the induction hypothesis, we have $P_{1, o_m} = P_{1, o_m}'$ for any $m<k$. Hence, $P_{1,o_k}' \geq P_{1,o_k}$. Agent $1$ is consuming $o_k$ during the whole interval $[t_{k-1}, \tau(o_k))$ under $R$ but during the sub interval of $[t_{k-1}, \tau'(o_k))$ under $R'$. Therefore, $\tau(o_k) \leq \tau'(o_k)$.

    We claim that for all $t \in [0,\tau(o_k))$ and all agents $j$, $j\neq 1$, we have:
    \begin{equation}\label{equ:claim1}
    j\in N(o_k,t) \Rightarrow j \in N'(o_k,t).
    \end{equation}
    Suppose for the sake of contradiction that there is an agent $j$, $j\neq 1$, and a time $t$, $0 \leq t < \tau(o_k)$ such that $j\in N(o_k,t)$ and $j\in N'(\hat o_1,t)$, for some item $\hat o_1 \neq o_k \in D_i$. As $t<\tau(o_k)\leq \tau'(o_k)$, $o_k$ is available at time $t$ under $R'$. Then $\hat o_1 \succ_j o_k$ because the preferences of agent $j$ are same under $R$ and $R'$. So $\hat o_1$ is unavailable at $t$ under $R$. This says $\tau(\hat o_1) \leq t < \tau'(\hat o_1)$. Let $A = \{o \in D_i: o \neq o_k \wedge \tau(o) < \tau'(o)\}$, we have $\hat o_1\in A$ and $A \neq \emptyset$. We pick $\hat o_2$ in $A$ such that $\tau(\hat o_2) = \min \{\tau(o): o\in A\}$. Note that $\tau(\hat o_2) < \tau(o_k)$ because $\tau(\hat o_1) < \tau(o_k)$ and $\hat o_1 \in A$. We claim that at some $t<\tau(\hat o_2)$, there is an agent $\hat j$ such that $\hat j \in N(\hat o_2,t)$ and $\hat j \notin N'(\hat o_2,t)$. Otherwise we have $N(\hat o_2,t) \subseteq N'(\hat o_2,t)$ for $t \in [0, \tau(\hat o_2))$. And there is some sub interval of $[\tau(\hat o_2), \tau'(\hat o_2))$ such that $n'(\hat o_2,t) \geq 1$ for any $t$ in the sub interval since $\tau'(\hat o_2) = sup\{t:n'(\hat o_2,t) \geq 1\}$. Hence
    \begin{equation}\label{equ:intless}
    \int_0^{\tau(\hat o_2)} n(\hat o_2,t) dt < \int_0^{\tau'(\hat o_2)} n'(\hat o_2,t) dt,
    \end{equation}
    a contradiction with
    \begin{equation}\label{equ:intequ}
    \int_0^{\tau(\hat o_2)} n(\hat o_2,t) dt = 1 = \int_0^{\tau'(\hat o_2)} n'(\hat o_2,t) dt.
    \end{equation}
    Note that $\hat j \neq 1$ because agent $1$ is consuming $o_k$ during $[t_{k-1}, \tau(o_k))$ under $R$. Suppose $\hat j \in N(\hat o_2,t)$ and $\hat j \in N'(\hat o_3,t)$, for some $t < \tau(\hat o_2)$ and $\hat o_3\in D_i$. As $\hat o_2$ is available at $t$ under $R'$ (because $t<\tau(\hat o_2)<\tau'(\hat o_2)$), we have $\hat o_3 \succ_{\hat j} \hat o_2$. Hence $\hat o_3$ is unavailable at $t$ under $R$ ($\hat j \neq 1$ and her preferences are same under $R$ and $R'$). Therefore, $\tau(\hat o_3) < t < \tau'(\hat o_3)$ and $\tau(\hat o_3) < t < \tau(\hat o_2)$, a contradiction with $\tau(\hat o_2) = \min \{\tau(o): o\in A\}$.

    Thus we have proved claim (\ref{equ:claim1}) is true. Suppose that agent $1$ consumes $o$ such that $o \neq o_k$ in the sub interval of $[t_{k-1},t_k]$ under $R'$. Then $o_k$ is available after $t_k$. If $t_k = 1$, we have a contradiction. If $t_k < 1$, we have $P_{1,o_k} < 1$ and $o_k$ is allocated to at least two agents. Since agents continue consuming an item until the item is exhausted once they start to consume the item in \eps{} under independent CP-nets by Claim~\ref{cl:orderinesp}. There is at least one agent except agent $1$ consumes $o_k$ just before $t_k$ under $R$ and $R'$ and the agent continues consuming $o_k$ after $t_k$ under $R'$. In that case, we have $P'_{1,o_k} = 1 - \sum_{j=2}^n P'_{j,o_k} < 1 - \sum_{j=2}^n P_{j,o_k} = P_{1,o_k}$, a contradiction with $P'_{1,o_k} \geq P_{1,o_k}$. Therefore, we have shown $N(o_k,t) \subseteq N'(o_k,t)$ for $t \in [0, \tau(o_k))$. If $\tau(o_k) < \tau'(o_k)$, we would have a contradiction like (\ref{equ:intless}) and (\ref{equ:intequ}). Hence $\tau(o_k) = \tau'(o_k)= t_k$. Since $P_{1,o_k}' \geq P_{1,o_k}$ and $P_{1,o_k} = \tau(o_k) - t_{k-1}$, agent $1$ is consuming $o_k$ during the whole interval $[t_{k-1}, \tau'(o_k))$ under $R'$ and $P_{1,o_k}' = P_{1,o_k}$. Since  the preferences of all agents except agent $1$ keep same under $R$ and $R'$,  $S(j,t,D_i) = S'(j,t,D_i)$ for any $j\in N$ and $t\in [t_{k-1},t_k)$.

    By induction, we have $S(j,t,D_i) = S'(j,t,D_i)$ for any $j \leq N$, $t\in [0,1)$ and $D_i \in T$. Hence, $P(1) = P'(1)$.

\end{proof}



\end{document}